\pdfoutput=1
\documentclass[11pt]{article}

\usepackage{EMNLP2022}

\usepackage{times}
\usepackage{latexsym}

\usepackage{bm}
\usepackage{graphicx}
\usepackage{subfigure}
\usepackage{amssymb}
\usepackage{amsmath}
\usepackage{amsthm}
\usepackage{xcolor}
\usepackage{multirow}
\usepackage{enumitem}

\usepackage{multirow}

\usepackage{algorithm}
\usepackage{algpseudocode}

\newtheorem{theorem}{Theorem}

\newcommand{\ftexttt}[1]{\texttt{\frenchspacing#1}}

\algnewcommand\algorithmicforeach{\textbf{for each}}
\algdef{S}[FOR]{ForEach}[1]{\algorithmicforeach\ #1\ \algorithmicdo}

\usepackage[T1]{fontenc}
\usepackage[utf8]{inputenc}

\usepackage{microtype}

\usepackage{inconsolata}

\title{MetaASSIST: Robust Dialogue State Tracking with Meta Learning}

\author{Fanghua Ye$^{\dagger}$ \space\space Xi Wang$^{\dagger}$ \space\space Jie Huang$^{\ddagger}$ \space\space Shenghui Li$^{\S}$ \space\space Samuel Stern$^{\P}$ \space\space Emine Yilmaz$^{\dagger}$ \\
        $^{\dagger}$University College London, UK \\ $^{\ddagger}$University of Illinois at Urbana-Champaign, USA \\ $^{\S}$Uppsala University, Sweden \\ $^{\P}$Affiniti AI, London, UK \\
        \texttt{\{fanghua.ye.19, xi-wang, emine.yilmaz\}@ucl.ac.uk} \\
         \texttt{jeffhj@illinois.edu, shenghui.li@it.uu.se, samuel.stern@affiniti.ai}}
\begin{document}
\maketitle
\begin{abstract}
Existing dialogue datasets contain lots of noise in their state annotations. Such noise can hurt model training and ultimately lead to poor generalization performance. A general framework named ASSIST has recently been proposed to train robust dialogue state tracking (DST) models. It introduces an auxiliary model to generate pseudo labels for the noisy training set. These pseudo labels are combined with vanilla labels by a \textit{common fixed} weighting parameter to train the primary DST model. Notwithstanding the improvements of ASSIST on DST, tuning the weighting parameter is challenging. Moreover, a single parameter shared by all slots and all instances may be suboptimal. To overcome these limitations, we propose a meta learning-based framework MetaASSIST to adaptively learn the weighting parameter. Specifically, we propose three schemes with varying degrees of flexibility, ranging from slot-wise to both slot-wise and instance-wise, to convert the weighting parameter into learnable functions. These functions are trained in a meta-learning manner by taking the validation set as meta data. Experimental results demonstrate that all three schemes can achieve competitive performance. Most impressively, we achieve a state-of-the-art joint goal accuracy of $80.10\%$ on MultiWOZ 2.4.
% \footnote{Code and data are in the supplementary materials and will be released as open-source after the review process.} 
\end{abstract}

\section{Introduction}

Task-oriented dialogue systems have recently become a hot research topic. They act as digital personal assistants, helping users with various tasks such as hotel bookings, restaurant reservations, and weather checks. Dialogue state tracking (DST) is recognized as a core task of the dialogue manager. Its goal is to keep track of users' intentions at each turn of the dialogue \citep{mrksic-etal-2017-neural, rastogi2020towards}. Tracking the dialogue state accurately is of significant importance, as the state information will be fed into the dialogue policy learning module to determine the next system action to perform \citep{manotumruksa-etal-2021-improving-dialogue}. In general, the dialogue state is represented as a set of (\textit{slot}, \textit{value}) pairs \citep{henderson-etal-2014-second, budzianowski-etal-2018-multiwoz}. The slots for a particular task or domain are predefined (e.g., ``\textit{hotel-name}''). Their values are extracted from the dialogue context.

So far, a great variety of DST models have been proposed \citep{wu-etal-2019-transferable, campagna-etal-2020-zero, balaraman-etal-2021-recent, lee-etal-2021-dialogue, guo-etal-2022-beyond, shin-etal-2022-dialogue, wang2022luna}. These models assume that all state labels provided in the dataset are correct, without considering the effect of label noise. However, dialogue state annotations are error-prone, especially considering that most dialogue datasets (e.g., MultiWOZ \citealp{budzianowski-etal-2018-multiwoz}) are collected through crowdsourcing. The presence of label noise may impair model training and lead to poor generalization performance of the trained model, as deep neural models can easily overfit noisy training data \citep{zhang2021understanding}.

\begin{figure}[t]
  \centering
  \includegraphics[width=1.0\linewidth]{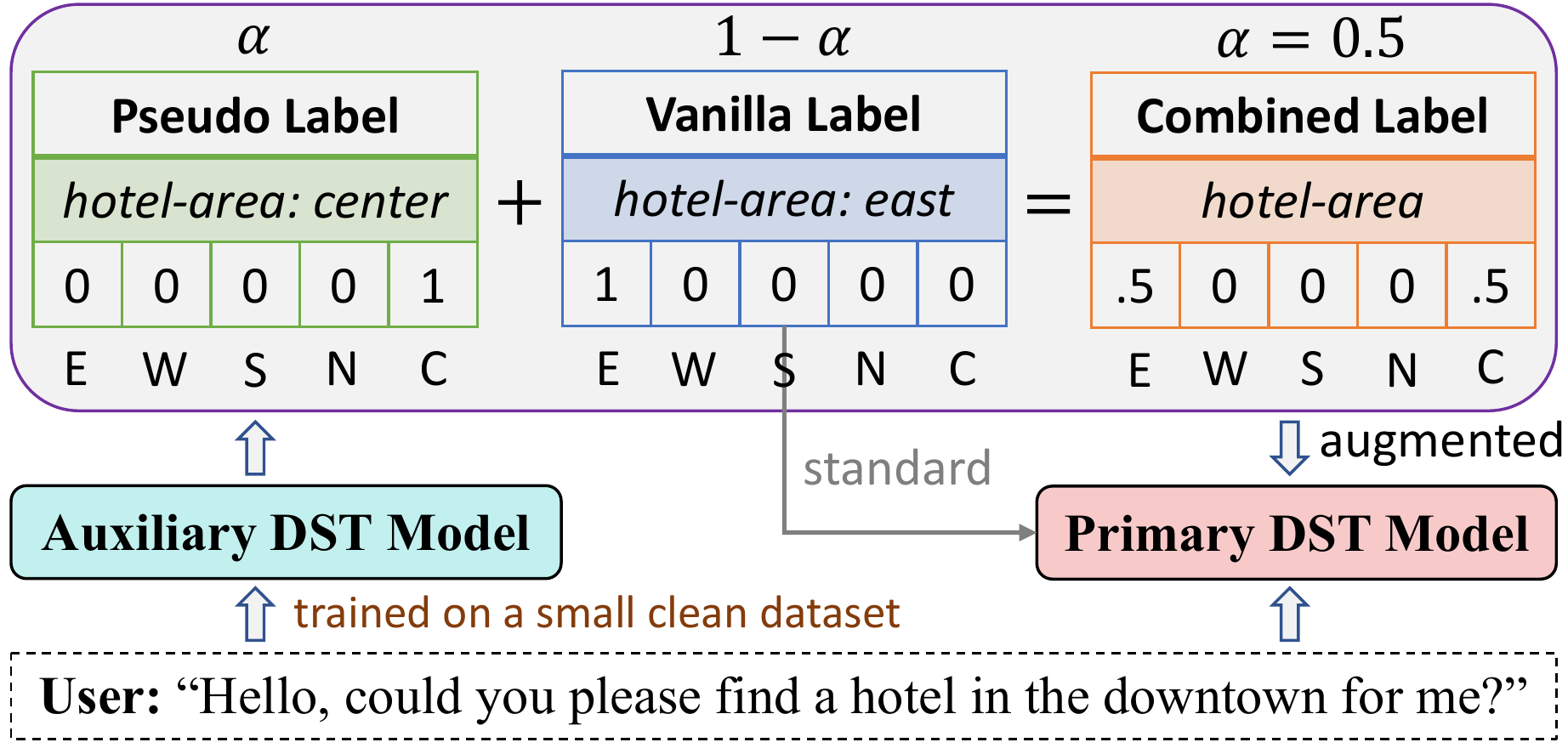}
%   \vspace*{-0.5cm}
  \caption{The structure of ASSIST and MetaASSIST. Both frameworks utilize soft labels obtained by linearly combining pseudo labels (one-hot) and vanilla labels (one-hot) using a weighting parameter $\alpha$ to enhance the training process compared to standard training that only relies on vanilla noisy labels. ASSIST adopts a single $\alpha$ shared by all slots and all training samples, while MetaASSIST uses slot-wise (and instance-wise) $\alpha$s.}
  \label{fig:framework}
%   \vspace{-0.4cm}
\end{figure}

% that is

In order to robustly train DST models from noisy labels, \citet{ye-etal-2022-assist} proposed a general framework dubbed ASSIST, which augments the standard model training procedure with a small clean dataset. As shown in Figure~\ref{fig:framework}, ASSIST first trains an auxiliary model on the small clean dataset and applies this model to generate pseudo labels for each sample in the noisy training set. Then, it linearly combines the pseudo labels and vanilla labels to train the primary model. Both theoretically and empirically, ASSIST has been shown to be effective in reducing the impact of label noise.

However, ASSIST adopts a common weighting parameter to combine the pseudo labels and vanilla labels for all slots and all training samples, which is suboptimal. In reality, different slots tend to have different noise rates \citep{eric-etal-2020-multiwoz}, indicating that the weighting parameter should be slot-wise. On the other hand, different training samples may also require different weighting parameters, since whether pseudo labels or vanilla labels should be preferred is highly dependent on specific training instances. Furthermore, the weighting parameter is considered a hyperparameter and thus needs to be carefully tuned on each dataset.

To address the aforementioned limitations of ASSIST, we propose MetaASSIST, a meta learning-based general framework that supports automatically learning slot-wise (and instance-wise) weighting parameters. Specifically, our contributions are:
\begin{itemize}
    \item We propose three different schemes for transforming the weighting parameters into learnable functions. These schemes have varying degrees of flexibility, ranging from slot-wise to both slot-wise and instance-wise. 
    
    \item We propose to train these learnable functions through a meta-learning paradigm that takes the validation set as meta data and adaptively adjusts the parameters of each learnable function (as a result, the weighting parameters) by reducing the validation loss.
    
    \item We conduct extensive experiments to test the effectiveness of the proposed three schemes. All of them achieve superior performance. For the first time, we achieve over $80\%$ joint goal accuracy on MultiWOZ 2.4 \citep{ye2021multiwoz}.
\end{itemize}

\section{Preliminaries}

In task-oriented dialogue systems, the DST module transforms users' goals or intentions expressed in unstructured natural languages into structured state representations (e.g., a series of slot-value pairs). The state representations are continually updated in each round of the user-system interactions.

% Typically, the dialogue state is represented as a set of predefined slots and their corresponding values, i.e., (\textit{slot}, \textit{value}) pairs \citep{henderson-etal-2014-second, budzianowski-etal-2018-multiwoz}. The slot values are extracted from the dialogue context and the state representations are continually updated in each round of the user-system interactions. 

\subsection{Problem Statement}

More formally, we symbolize a dialogue of $T$ turns as $\mathcal{X} = \{(R_1, U_1), \dots, (R_T, U_T)\}$, where $R_t$ and $U_t$ denote the system response and user utterance at turn $t$ ($1 \leq t\leq T$), respectively. We adopt $\mathcal{X}_t$ to represent the  dialogue context from the first turn to the $t$-th turn, i.e., $\mathcal{X}_t = \{(R_1, U_1), \dots, (R_t, U_t)\}$. Further, let $\mathcal{S}$ denote the set of all the predefined slots and $\mathcal{B}_t = \{(s, v_t)|s \in \mathcal{S}\}$ the dialogue state at turn $t$. Here, $v_t$ is the corresponding value of slot $s$ at turn $t$. Then, the DST problem is defined as learning a dialogue state tracker $\mathcal{F}: \mathcal{X}_t \rightarrow \mathcal{B}_t$.

As discussed earlier, annotating dialogue states via crowdsourcing is prone to incorrect and inconsistent labels. These noisy annotations are likely to adversely affect model training. We denote the noisy state annotations as $\tilde{\mathcal{B}}_t = \{(s, \tilde{v}_t)|s \in \mathcal{S}\}$, where $\tilde{v}_t$ is the noisy label of slot $s$ at turn $t$. In this work, $\tilde{\mathcal{B}}_t$ refers to the labels provided in the dataset and $\mathcal{B}_t$ refers to the unknown true state annotations. As pointed out by \citet{ye-etal-2022-assist}, existing DST approaches are only able to learn a suboptimal dialogue state tracker $\tilde{\mathcal{F}} : \mathcal{X}_t \rightarrow \tilde{\mathcal{B}}_t$ rather than the optimal dialogue state tracker $\mathcal{F}: \mathcal{X}_t \rightarrow \mathcal{B}_t$. Aiming at learning a strong dialogue state tracker $\mathcal{F}^*$ to better approximate $\mathcal{F}$, \citet{ye-etal-2022-assist} proposed a general framework ASSIST that supports training DST models robustly from noisy labels.

\subsection{Overview of ASSIST}

ASSIST assumes that a small clean dataset is available. Based on this assumption, it firstly trains an auxiliary model on the clean dataset. Then, it leverages the trained model to generate pseudo labels for each sample in the large noisy training set. The generated pseudo labels are expected to be a good complement to the vanilla noisy labels. Therefore, combining the two types of labels has the potential to reduce the influence of noisy labels when training the primary model. 

% As shown in Figure~\ref{fig:framework}, the framework of ASSIST consists of an auxiliary model and a primary model. The auxiliary model is firstly trained on a small clean dataset and then leveraged to generate pseudo labels for each sample in the large noisy training set. It has been shown that the generated pseudo labels are a good complement to the vanilla noisy labels, which indicates that combining the two types of labels has the potential to reduce the influence of noisy labels when training the primary model. 

Denote the generated pseudo state annotations as $\breve{\mathcal{B}}_t = \{(s, \breve{v}_t)|s \in \mathcal{S}\}$, where $\breve{v}_t$ represents the pseudo label of slot $s$ at turn $t$. Within the framework of ASSIST,  the primary model is required to predict  $\breve{\mathcal{B}}_t$ and $\tilde{\mathcal{B}}_t$ concurrently during the training process. In other words, the target of model training turns into learning a dialogue state tracker $\mathcal{F}^* : \mathcal{X}_t \rightarrow C(\breve{\mathcal{B}}_t, \tilde{\mathcal{B}}_t)$, where $C(\breve{\mathcal{B}}_t, \tilde{\mathcal{B}}_t)$ denotes a combination of $\breve{\mathcal{B}}_t$ and $\tilde{\mathcal{B}}_t$. There can be different methods to combine the generated pseudo labels and vanilla noisy labels. The most straightforward way is to combine them linearly, which is also the strategy adopted in ASSIST. The linearly combined label of slot $s$ at turn $t$ is formulated as:
\begin{equation} \label{eq:com}
    \bm{v}^c_t = \alpha \breve{\bm{v}}_t + (1 - \alpha) \tilde{\bm{v}}_t,
\end{equation}
where $\breve{\bm{v}}_t$ and $\tilde{\bm{v}}_t$ are the \textit{one-hot} vector representation of the pseudo label $\breve{v}_t$ and vanilla noisy label $\tilde{v}_t$, respectively. The parameter $\alpha$ $(0 \le \alpha \le 1)$ is employed to control the weights of $\breve{\bm{v}}_t$ and $\tilde{\bm{v}}_t$.

Let $p(\breve{v}_t | \mathcal{X}_t, s)$ denote the likelihood of $\breve{v}_t$ and $p(\tilde{v}_t | \mathcal{X}_t, s)$ the likelihood of $\tilde{v}_t$. Then, the likelihood of the combined label $\bm{v}^c_t$ is calculated as:
\begin{equation}
    p(\bm{v}^c_t|\mathcal{X}_t, s) = p(\breve{v}_t | \mathcal{X}_t, s)^{\alpha}  p(\tilde{v}_t | \mathcal{X}_t, s)^{(1 - \alpha)}.
\end{equation}
Based on this formula, the training objective of the primary model can be derived as follows:
\begin{equation} \label{eq:obj}
\begin{aligned}
    \mathcal{L} =& \frac{1}{|\mathcal{D}_n||\mathcal{S}|}\sum_{\mathcal{X}_t \in \mathcal{D}_n} \sum_{s \in \mathcal{S}} - \log p(\bm{v}^c_t|\mathcal{X}_t, s) \\
    =&  \frac{\alpha}{|\mathcal{D}_n||\mathcal{S}|} \sum_{\mathcal{X}_t \in \mathcal{D}_n} \sum_{s \in \mathcal{S}} - \log p(\breve{v}_t | \mathcal{X}_t, s) + \\
    & \frac{(1 - \alpha)}{|\mathcal{D}_n||\mathcal{S}|} \sum_{\mathcal{X}_t \in \mathcal{D}_n} \sum_{s \in \mathcal{S}} -\log p(\tilde{v}_t | \mathcal{X}_t, s),
\end{aligned}
\end{equation}
where $\mathcal{D}_n$ represents the noisy training set.

% $C(\breve{\mathcal{B}}_t, \tilde{\mathcal{B}}_t) = \{(s, \bm{v}^c_t)|s \in \mathcal{S}\}$.

\section{MetaASSIST: A Meta Learning-Based Version of ASSIST}

Equations~\eqref{eq:com} and \eqref{eq:obj} show that a single $\alpha$ is shared by all slots when combining the pseudo labels and vanilla labels. This is suboptimal, as  the ratio of the noise rate of pseudo labels to that of vanilla labels tends to be different for different slots. When the vanilla labels have higher quality than the generated pseudo labels, $\alpha$ should be set to a small value; otherwise, a large $\alpha$ should be used. This implies that setting $\alpha$ to different values for different slots can help train the primary model more robustly. In the following, we first theoretically show that the combined labels obtained via slot-wise weighting parameters instead of a common one can better approximate the unknown true labels. Then, we elaborate on the proposed framework MetaASSIST.

\subsection{Theoretical Justification}

Following \citep{ye-etal-2022-assist}, we employ the mean squared loss to define the mean approximation error of any corrupted labels $\ddot{\bm{v}}_t$ to their corresponding unknown true labels $\bm{v}_t$, as formularized below:
\begin{equation} \label{eq:err}
     Y_{\ddot{\bm{v}}} = \frac{1}{|\mathcal{D}_n||\mathcal{S}|}\sum_{\mathcal{X}_t \in \mathcal{D}_n} \sum_{s \in \mathcal{S}} E_{\mathcal{D}_c} [ \Vert \ddot{\bm{v}}_t - \bm{v}_t \Vert^2_2].
\end{equation}
Here, $\mathcal{D}_c$ refers to the small clean dataset. Both $\ddot{\bm{v}}_t$ and $\bm{v}_t$ are the vector representations of labels.

Let $\alpha_s$ be the slot-wise weighting parameter for slot $s$. We utilize $\bm{v}^s_t$ to denote the combined label obtained by replacing $\alpha$ with $\alpha_s$ in Eq.~\eqref{eq:com}. Thus,
\begin{equation} \label{eq:slot-wise}
    \bm{v}^s_t = \alpha_s \breve{\bm{v}}_t + (1 - \alpha_s) \tilde{\bm{v}}_t.
\end{equation}
Same as $\alpha$, $\alpha_s$ is also bounded between 0 and 1.

Substituting the corrupted labels $\ddot{\bm{v}}_t$ in Eq.~\eqref{eq:err} with $\bm{v}^s_t$ and $\bm{v}^c_t$, we have the following theorem:
\begin{theorem} \label{thm}
 The optimal mean approximation error with respect to the combined labels $\bm{v}^s_t$ derived from slot-wise weighting parameters $\alpha_s$ is smaller than or equal to that of the combined labels $\bm{v}^c_t$ derived from a shared weighting parameter $\alpha$, i.e., 
 \begin{equation*}
     \min_{\alpha_s} Y_{\bm{v}^s} \leq \min_{\alpha} Y_{\bm{v}^c}.
 \end{equation*}
\end{theorem}
\begin{proof}
The conclusion is obvious as we can replace $\alpha_s$ with $\alpha$ if $Y_{\bm{v}^c} < Y_{\bm{v}^s}$, but not vice versa. 
\end{proof}
% Check Appendix~\ref{sec:appendixproof} for the proof. 

\subsection{Slot-Wise Weighting Parameters as Meta Learnable Functions}

In the framework of ASSIST, $\alpha$ is treated as a hyperparameter. It needs to be meticulously tuned in the training phase so as to help the primary model achieve the best performance. Although it is feasible to tune a single parameter $\alpha$, it would become extremely painful to tune all the slot-wise parameters. This is because multi-domain dialogues can have dozens of or even hundreds of slots (e.g., there are 37 slots in the MultiWOZ dataset \citealp{eric-etal-2020-multiwoz}). To circumvent the troublesome step of tuning each slot-wise parameter $\alpha_s$ of slot $s$, we propose to learn all these parameters automatically via meta learning \citep{hospedales2021meta}. 

Specifically, we propose three different schemes to cast the slot-wise weighting parameters as learnable functions, which are described in detail below:
\begin{itemize} [label={}, leftmargin=0pt]
    \item \textbf{Scheme One (S1):} The first scheme assumes that the parameter $\alpha_s$ is fully independent of the dialogue context $\mathcal{X}_t$. As a consequence of this assumption, all the training samples will share the same $\alpha_s$ for slot $s$. Given that the parameter $\alpha_s$ is restricted to fall in the range of 0 to 1, it is tricky to learn it by gradient-based optimizers. In our implementation, we introduce an unconstrained learnable parameter $w_s$ and regard $\alpha_s$ as a Sigmoid function of $w_s$:
    \begin{equation} 
       \alpha_s = f_1(w_s) = \verb|Sigmoid|(w_s).
    \end{equation}
    As thus, the parameter $w_s$ rather than $\alpha_s$ will be directly optimized during the training process.
    
    %whether or not the vanilla and pseudo annotations are noisy may highly depend on the dialogue context. On one hand, when users provide information explicitly in the dialogue context, the annotators are likely to annotate the states correctly. By contrast, if the dialogue context involves coreference or omission or even commonsense reasoning, it becomes challenging for the annotators to provide accurate and consistent annotations. On the other hand, 
    
    \item \textbf{Scheme Two (S2):} Apart from being slot-wise, the second scheme assumes that the parameter $\alpha_s$ should also be relevant to the dialogue context $\mathcal{X}_t$ (i.e., instance-wise). This assumption is of practical significance, as whether the vanilla labels or the pseudo labels should be preferred may vary across the training samples. In order to make $\alpha_s$ instance-wise, we first construct a five-dimensional feature vector based on the loss values of both vanilla labels and pseudo labels, as shown below:
    \begin{equation} \label{eq:feature}
       \bm{h}_s = [\tilde{l}_s, \breve{l}_s, \tilde{l}_s - \breve{l}_s, \breve{l}_s - \tilde{l}_s, \tilde{l}_s + \breve{l}_s],
    \end{equation}
    where $\tilde{l}_s$ and $\breve{l}_s$ correspond to the loss value of the vanilla label $\tilde{v}_t$ and pseudo label $\breve{v}_t$ of slot $s$ associated with the dialogue context $\mathcal{X}_t$, respectively. To be more specific, $\tilde{l}_s$ and $\breve{l}_s$ are calculated as follows:
    \begin{align}
        \tilde{l}_s &= - \log p(\tilde{v}_t | \mathcal{X}_t, s), \\
        \breve{l}_s &= - \log p(\breve{v}_t | \mathcal{X}_t, s).
    \end{align}
    We then utilize an MLP network \citep{rumelhart1986learning} with a single hidden layer followed by the Sigmoid activation function to learn $\alpha_s$:
    \begin{equation} 
       \alpha_s = f_2(\bm{h}_s) = \verb|Sigmoid|(\verb|MLP|(\bm{h}_s)).
    \end{equation}
    
    \item \textbf{Scheme Three (S3):} The first and second schemes require that the weights of the pseudo label $\breve{v}_t$ and vanilla label $\tilde{v}_t$ of each slot in each training sample must add up to 1. In reality, however, both $\breve{v}_t$ and $\tilde{v}_t$ can be incorrect for some training samples, in which case, it is beneficial to assign small weights to both labels. In the third scheme, we remove the constraint on the sum and adopt two weighting parameters to combine the pseudo labels and vanilla labels. The combined label $\bm{v}^s_t$ is given by:
    \begin{equation} \label{eq:slot-wise2}
    \bm{v}^s_t = \breve{\alpha}_s \breve{\bm{v}}_t + \tilde{\alpha}_s \tilde{\bm{v}}_t.
    \end{equation}
    We learn $\breve{\alpha}_s$ and $\tilde{\alpha}_s$ ($0\leq \breve{\alpha}_s, \tilde{\alpha}_s \leq 1$) in the same way as how $\alpha_s$ is learned in the second scheme\footnote{With no constraint on the sum, $\breve{\alpha}_s$ and $\tilde{\alpha}_s$ are decoupled. One may argue that both parameters should be relevant to only the loss value of their corresponding label. In view of this, we also evaluated $\breve{\alpha}_s = f_3(\breve{l}_s)$ and $\tilde{\alpha}_s = f'_3(\tilde{l}_s)$. However, we found that the performance is much worse than using the loss value of the pseudo label and that of the vanilla label together.}:
    \begin{align}
        \breve{\alpha}_s &= f_3(\bm{h}_s) = \ftexttt{Sigmoid}(\ftexttt{MLP}(\bm{h}_s)), \\
        \tilde{\alpha}_s &= f'_3(\bm{h}_s) = \ftexttt{Sigmoid}(\ftexttt{MLP}(\bm{h}_s)).
    \end{align}
    It is noted that Eq.~\eqref{eq:slot-wise2} can be rewritten as:
    \begin{equation} \label{eq:slot-wise3}
    \bm{v}^s_t = (\breve{\alpha}_s + \tilde{\alpha}_s) \Big( \beta_s \breve{\bm{v}}_t + (1-\beta_s) \tilde{\bm{v}}_t\Big),
    \end{equation}
    where $\beta_s = \breve{\alpha}_s / (\breve{\alpha}_s + \tilde{\alpha}_s)$. Comparing Eq.~\eqref{eq:slot-wise3} to Eq.~\eqref{eq:slot-wise}, it can be seen that the main difference is that the combined label is further weighted by $\breve{\alpha}_s + \tilde{\alpha}_s$. This reweighting is expected to be able to discard the training samples whose pseudo labels and vanilla labels are both incorrect by adjusting $\breve{\alpha}_s + \tilde{\alpha}_s$ to be a small value. 
    
\end{itemize}

% From scheme one to scheme three, the flexibility of weighting parameters gradually increases. 

In schemes S2 and S3, the weighting parameters are both slot-wise and instance-wise. Compared to scheme S1 in which the weighting parameters are only slot-wise, adding the instance-wise flexibility can make the combined labels even more accurate in the optimal case. For example, when the pseudo label of slot $s$ in a training sample is correct while its vanilla label is wrong, the best $\alpha_s$ in scheme S2 will be 1.0, which leads to 0 approximation error.

\subsection{Learning Algorithm}

When training the primary model, besides its own parameters, the parameters of the learnable functions that are used to predict the weights also need to be optimized. Inspired by the common practice that the best model checkpoint is chosen according to the performance on the validation set, we decide to employ the validation set as meta data and then train the involved functions (i.e., $f_1$, $f_2$, $f_3$ and $f'_3$) in a meta-learning manner. 

For the sake of uniformly describing the learning processes of the three proposed schemes, we unify the combined label $\bm{v}^s_t$ as: 
\begin{equation}
    \bm{v}^s_t = f(\bm{w}_1) \breve{\bm{v}}_t + f'(\bm{w}_2) \tilde{\bm{v}}_t,
\end{equation}
where $\bm{w}_1$ and $\bm{w}_2$ are the parameters of the learnable functions. Note that for schemes S1 ans S2, $f'(\bm{w}_2) = 1 - f(\bm{w}_1)$\footnote{In this case, only $\bm{w}_1$ needs to be optimized. We keep $\bm{w}_2$ for ease of exposition, because it is required by scheme S3.}. Then, the training objective of the primary model is derived as:
\begin{equation}
\begin{aligned}
    &\mathcal{L}(\Theta) = \frac{1}{|\mathcal{D}_n||\mathcal{S}|}\sum_{\mathcal{X}_t \in \mathcal{D}_n} \sum_{s \in \mathcal{S}} - \log p(\bm{v}^s_t|\mathcal{X}_t, s) \\
    &= \frac{1}{|\mathcal{D}_n||\mathcal{S}|}\sum_{\mathcal{X}_t \in \mathcal{D}_n} \sum_{s \in \mathcal{S}} \big(f(\bm{w}_1)\breve{l}_s + f'(\bm{w}_2)\tilde{l}_s\big).
\end{aligned}
\end{equation}
Here, $\Theta$ represents the parameters of the primary model and is optimized by minimizing $\mathcal{L}(\Theta)$, i.e.,
\begin{equation} \label{eq:opttheta}
    \Theta^*(\bm{w}_1, \bm{w}_2) = \operatorname*{arg\,min}_\Theta \mathcal{L}(\Theta).
\end{equation}
The optimal parameters $\Theta^*(\bm{w}_1, \bm{w}_2)$ are expected to achieve the best performance on the validation set $\mathcal{D}_v$. Hence, we can optimize $\bm{w}_1$ and $\bm{w}_2$ in the following way:
\begin{equation} \label{eq:optw}
    \bm{w}^*_1, \bm{w}^*_2 = \operatorname*{arg\,min}_{\bm{w}_1, \bm{w}_2} \frac{\sum\limits_{\mathcal{X}_t \in \mathcal{D}_v} \sum\limits_{s \in \mathcal{S}} l^v_s(\Theta^*(\bm{w}_1, \bm{w}_2))}{|\mathcal{D}_v||\mathcal{S}|},
\end{equation}
where $l^v_s(\Theta^*(\bm{w}_1, \bm{w}_2)) = -\log p(v_t|\mathcal{X}_t, s)$ represents the loss of slot $s$ corresponding to the validation sample $\mathcal{X}_t$, calculated from the predictions of the primary model with parameters $\Theta^*(\bm{w}_1, \bm{w}_2)$.

\subsubsection*{Batch-Based Online Approximation}

As shown in Eqs.~\eqref{eq:opttheta} and \eqref{eq:optw}, two nested loops of optimization are required for calculating the optimal parameters ${\Theta}^*$, $\bm{w}^*_1$ and $\bm{w}^*_2$. Each single loop on the whole dataset can be fairly expensive. Following \citep{ren2018learning}, we adopt an online strategy to update $\Theta$, $\bm{w}_1$ and $\bm{w}_2$ alternately through a single optimization loop based on mini-batch data. Algorithm~\ref{alg:cap} summarizes the overall training procedure (including auxiliary model training).
\begin{algorithm}[h]
    \caption{Learning algorithm of MetaASSIST}\label{alg:cap}
    \small
    \begin{algorithmic}[1]
        \Require The small clean dataset $\mathcal{D}_c$, noisy training dataset $\mathcal{D}_n$, validation dataset $\mathcal{D}_v$, batch size $n$, $m$, $k$, and number of training steps for auxiliary and primary model $J_\mathcal{A}$, $J_\mathcal{P}$;
        \Ensure The parameters of the learnable functions $\bm{w}^{(J_\mathcal{P})}_1$, $\bm{w}^{(J_\mathcal{P})}_2$ and the parameters of the primary model $\Theta^{(J_\mathcal{P})}$;
        \State $\rhd$ \textbf{Auxiliary model training}
        \For{$j = 1, 2,\dots, J_\mathcal{A}$}
           \State $\mathcal{M}_c \leftarrow$ SampleMiniBatch($\mathcal{D}_c$, $n$);
           \State Update the auxiliary model on $\mathcal{M}_c$;
        \EndFor
        \State Apply the trained auxiliary model to generate pseudo state annotations $\breve{\mathcal{B}}_t$ for each $\mathcal{X}_t \in \mathcal{D}_n$;
      
        \State $\rhd$ \textbf{Primary model training}
        \State Initialize parameters $\Theta^{(0)}$, $\bm{w}^{(0)}_1$ and $\bm{w}^{(0)}_2$;
        \For{$j=1,2,\dots,J_\mathcal{P}$}
           \State $\mathcal{M}_n \leftarrow$ SampleMiniBatch($\mathcal{D}_n$, $m$);
           \State $\mathcal{M}_v \leftarrow$ SampleMiniBatch($\mathcal{D}_v$, $k$);
           \State \parbox[t]{\dimexpr\linewidth-\algorithmicindent}{$\hat{\Theta}^{(j)}(\bm{w}^{(j-1)}_1, \bm{w}^{(j-1)}_2) \leftarrow$ Update $\Theta^{(j-1)}$ on $\mathcal{M}_n$ using $\bm{v}^s_t = f(\bm{w}^{(j-1)}_1) \breve{\bm{v}}_t + f'(\bm{w}^{(j-1)}_2) \tilde{\bm{v}}_t$ as the label;}
           
           \State \parbox[t]{\dimexpr\linewidth-\algorithmicindent}{$\bm{w}^{(j)}_1, \bm{w}^{(j)}_2 \leftarrow$ Update $\bm{w}^{(j-1)}_1$ and $\bm{w}^{(j-1)}_2$ on $\mathcal{M}_v$ with loss values derived from $\hat{\Theta}^{(j)}(\bm{w}^{(j-1)}_1, \bm{w}^{(j-1)}_2)$;}
           
           \State \parbox[t]{\dimexpr\linewidth-\algorithmicindent}{$\Theta^{(j)} \leftarrow$ Update $\Theta^{(j-1)}$ on $\mathcal{M}_n$ again using the new label $\bm{v}^s_t = f(\bm{w}^{(j)}_1) \breve{\bm{v}}_t + f'(\bm{w}^{(j)}_2) \tilde{\bm{v}}_t$;}
        \EndFor
    \end{algorithmic}
\end{algorithm}

The procedure of training the primary model in MetaASSIST is similar to that of standard model training, except that three extra steps (lines 11-13) are added. This is because the optimal combined label $\bm{v}^s_t$ is unknown upon beginning. In Algorithm~\ref{alg:cap}, we choose to dynamically update $\bm{v}^s_t$ by adapting $\bm{w}_1$ and $\bm{w}_2$. At first, we use $\bm{w}^{(j-1)}_1$ and $\bm{w}^{(j-1)}_2$ to derive $\bm{v}^s_t$ and train the primary model on batch $\mathcal{M}_n$ for one step, which results in an interim model with parameters $\hat{\Theta}^{(j)}(\bm{w}^{(j-1)}_1, \bm{w}^{(j-1)}_2)$ (line 12). Then, we apply this interim model to the validation batch $\mathcal{M}_v$ and compute the validation loss. By lowering this loss (e.g., one-step optimization by SGD), we obtain the updated $\bm{w}^{(j)}_1$ and $\bm{w}^{(j)}_2$ (line 13). After that, we use $\bm{w}^{(j)}_1$ and $\bm{w}^{(j)}_2$ to update $\bm{v}^s_t$ and apply this new combined label to train the $(j-1)$-th step primary model on batch $\mathcal{M}_n$ again, which eventually leads to the updated primary model (line 14).

% of the $j$-th step

\section{Experimental Setup}

\subsection{Datasets}

We conduct experiments mainly on MultiWOZ 2.4 \citep{ye2021multiwoz}. It is the latest refined version of MultiWOZ 2.0 \citep{budzianowski-etal-2018-multiwoz}, a large-scale multi-domain task-oriented dialogue dataset consisting of over 10,000 dialogues spanning seven domains. The validation set and test set of MultiWOZ 2.4 have been carefully reannotated, while its training set remains the same as that of MultiWOZ 2.1 \citep{eric-etal-2020-multiwoz} and is therefore noisy. Following \citep{ye-etal-2022-assist}, we adopt the validation set as the small clean dataset. Thus, the validation set is used to train both the auxiliary model and the learnable functions. We also conduct experiments on MultiWOZ 2.0, whose validation set and test set have been replaced with the counterparts of MultiWOZ 2.4. Due to this change, we name the dataset MultiWOZ 2.0* in the following. The only difference between MultiWOZ 2.0* and MultiWOZ 2.4 is that the training set of the former is much noisier. 

% The MultiWOZ dataset contains over 10,000 dialogues across seven domains, including \textit{attraction}, \textit{hotel}, \textit{restaurant}, \textit{taxi}, \textit{train}, \textit{hospital} and \textit{police}. 

% \usepackage{multirow}
\begin{table*}[t]
\centering
\setlength{\tabcolsep}{1.75mm}
\begin{tabular}{c|c|c|ccc|ccc}
\hline
\multirow{2}{*}{\textbf{\begin{tabular}[c]{@{}c@{}}Primary\\ Model\end{tabular}}} & \multirow{2}{*}{\textbf{Framework}} & \multirow{2}{*}{\textbf{Scheme}} & \multicolumn{3}{c|}{\textbf{Validation}} & \multicolumn{3}{c}{\textbf{Test}} \\ \cline{4-9} 
 &  &  & \textbf{JGA(\%)} & \textbf{JTA(\%)} & \textbf{SA(\%)} & \textbf{JGA(\%)} & \textbf{JTA(\%)} & \textbf{SA(\%)} \\ \hline \hline
\multirow{6}{*}{SOM-DST} & \multirow{3}{*}{ASSIST} & $\alpha=0.0$ & 68.77 & 87.85 & 98.45 & 66.78 & 87.81 & 98.38 \\
 &  & $\alpha=1.0$ & 77.35 & 91.05 & 98.98 & 68.69 & 88.41 & 98.55 \\
 &  & $\alpha=0.4$ & 78.59 & 91.74 & 99.02 & 75.19 & 91.02 & 98.84 \\ \cline{2-9} 
 & \multirow{3}{*}{MetaASSIST} & S1 & \textbf{80.95} & \textbf{92.64} & \textbf{99.16} & 75.12 & 90.88 & 98.87 \\
 &  & S2 & 78.87 & 92.01 & 99.07 & \textbf{76.74} & \textbf{91.65} & \textbf{98.95} \\
 &  & S3 & 80.02 & 92.05 & 99.12 & 75.20 & 91.07 & 98.90 \\ \hline \hline
\multirow{6}{*}{STAR} & \multirow{3}{*}{ASSIST} & $\alpha=0.0$ & 74.33 & 90.26 & 98.86 & 74.84 & 90.77 & 98.92 \\
 &  & $\alpha=1.0$ & 80.27 & 90.29 & 99.17 & 71.01 & 86.31 & 98.69 \\
 &  & $\alpha=0.4$ & 82.68 & 92.93 & 99.26 & 79.41 & 91.86 & 99.14 \\ \cline{2-9} 
 & \multirow{3}{*}{MetaASSIST} & S1 & \textbf{83.40} & 93.03 & \textbf{99.32} & 77.80 & 90.85 & 99.05 \\
 &  & S2 & 83.03 & 93.19 & 99.30 & \textbf{80.10} & \textbf{92.02} & \textbf{99.16} \\
 &  & S3 & 83.13 & \textbf{93.45} & 99.30 & 79.37 & 91.84 & 99.12 \\ \hline \hline
\multirow{6}{*}{AUX-DST} & \multirow{3}{*}{ASSIST} & $\alpha=0.0$ & 72.47 & 89.57 & 98.78 & 70.37 & 89.31 & 98.67 \\
 &  & $\alpha=1.0$ & 81.30 & 90.68 & 99.22 & 70.68 & 86.82 & 98.68 \\
 &  & $\alpha=0.4$ & \textbf{83.97} & \textbf{93.49} & \textbf{99.33} & 78.14 & 91.03 & 99.07 \\ \cline{2-9} 
 & \multirow{3}{*}{MetaASSIST} & S1 & 83.89 & 93.41 & \textbf{99.33} & 77.25 & 91.16 & 99.04 \\
 &  & S2 & 80.97 & 92.18 & 99.21 & 78.38 & 91.57 & 99.06 \\
 &  & S3 & 81.99 & 92.88 & 99.24 & \textbf{78.57} & \textbf{92.09} & \textbf{99.08} \\ \hline
\end{tabular}
\caption{Performance comparison on MultiWOZ 2.4. For ASSIST, $\alpha=0.0$ means that only the vanilla labels are used to train the primary model. $\alpha=1.0$ means that only the generated pseudo labels are used. $\alpha=0.4$ is the best common weighting parameter found in \citep{ye-etal-2022-assist}. All schemes in MetaASSIST use both types of labels.}
\label{tab:main}
\end{table*}

% all three schemes

% 81.16 & 92.35 & 99.21 & 78.50 & 91.60 & 99.05

\subsection{Evaluation Metrics}

We adopt Joint Goal Accuracy (\textbf{JGA}), Joint Turn Accuracy (\textbf{JTA}) and Slot Accuracy (\textbf{SA}) as evaluation metrics. JGA is the primary metric for DST. It refers to the ratio of dialogue turns of which the entire state is correctly predicted. JTA is defined as the ratio of dialogue turns in which the values of all active slots are correctly predicted. A slot is said to be active if its value needs to be updated. SA considers only slot-level information and is calculated as the average of all individual slot accuracies.

\subsection{Auxiliary and Primary Models}

% For a fair comparison, our proposed framework
We use the same auxiliary and primary models as ASSIST to assess the effectiveness of MetaASSIST. Since the clean dataset is small, a simple auxiliary model \textbf{AUX-DST} was specially designed to avoid overfitting \citep{ye-etal-2022-assist}. AUX-DST leverages slot-token attention to extract slot-specific information and selects the value that best matches this information as prediction. It is also adopted as one primary model. The other primary models considered are: 1) \textbf{SOM-DST} \citep{kim-etal-2020-efficient}, an open vocabulary method that regards the dialogue state as a fixed-sized memory and selectively overwrites this memory with new values; and 2) \textbf{STAR} \citep{ye2021slot}, an ontology-based method that uses a stacked slot self-attention mechanism to learn the correlations amongst slots automatically.\footnote{Code is available at \url{https://github.com/smartyfh/DST-MetaASSIST}}

\section{Results and Discussion}

\subsection{Main Results}

Table~\ref{tab:main} shows the performance of the three primary models on MultiWOZ 2.4  trained using ASSIST and our proposed framework MetaASSIST. We observe that all three schemes in MetaASSIST substantially improve the performance of the primary models on the test set compared to training with only vanilla labels ($\alpha=0.0$) or only pseudo labels ($\alpha=1.0$). This observation indicates that the proposed schemes are effective in learning appropriate weighting parameters for combining pseudo labels and vanilla labels. Further, we observe that scheme S2 consistently outperforms ASSIST with the best common weighting parameter ($\alpha=0.4$), except the slot accuracy of AUX-DST. For example, STAR achieves $80.10\%$ joint goal accuracy when using scheme S2 to learn the weighting parameters. Table~\ref{tab:2.0res} presents the performance of SOM-DST trained on MultiWOZ 2.0*. It also shows that scheme S2 achieves better results.

% \usepackage{multirow}
\begin{table}[]
\centering
\setlength{\tabcolsep}{1.2mm}
\begin{tabular}{c|c|ccc}
\hline
\textbf{Frame} & \textbf{Scheme} & \textbf{JGA(\%)} & \textbf{JTA(\%)} & \textbf{SA(\%)} \\ \hline
\multirow{3}{*}{ASSIST} & $\alpha=0.0$ & 45.14 & 77.86 & 96.71 \\
 & $\alpha=1.0$ & 67.06 & 87.95 & 98.47 \\
 & $\alpha=0.6$ & 70.83 & 89.14 & 98.61 \\ \hline
\multirow{3}{*}{\begin{tabular}[c]{@{}c@{}}Meta\\ ASSIST\end{tabular}} & S1 & 70.18 & 88.69 & 98.60 \\
 & S2 & \textbf{71.46} & \textbf{89.35} & \textbf{98.65} \\
 & S3 & 70.48 & 88.84 & 98.60 \\ \hline
\end{tabular}
\caption{Performance comparison on MultiWOZ 2.0*'s test set by taking SOM-DST as the primary model. On this dataset, the best value of $\alpha$ for ASSIST is 0.6.}
\label{tab:2.0res}
% \vspace*{-0.5cm}
\end{table}

\begin{figure*}[t!]
	\centering
	\subfigure[{Schemes S1 and S2}]{
		\includegraphics[width=0.65\columnwidth]{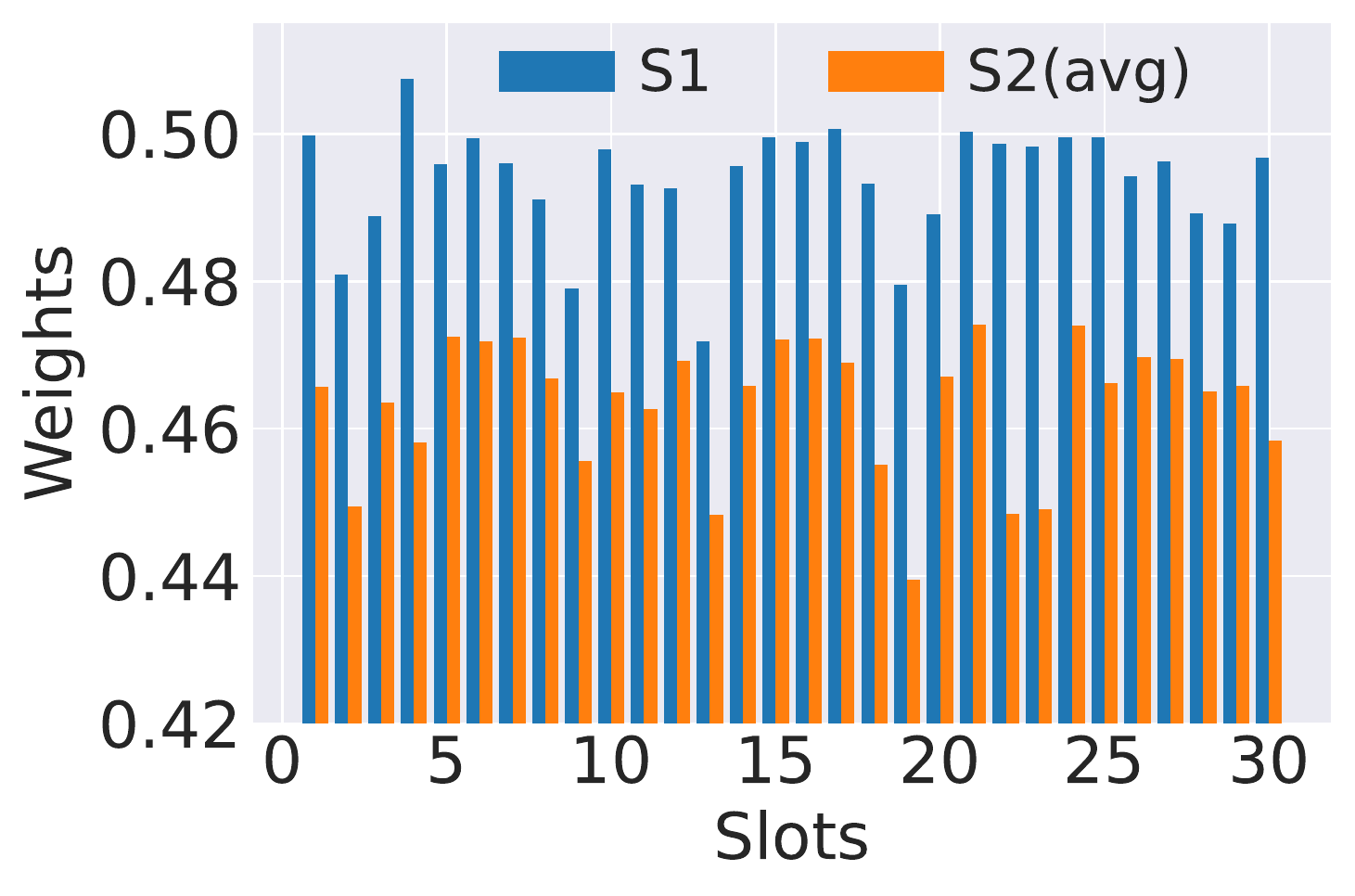}
	} \hfill
	\subfigure[{Scheme S2}]{
		\includegraphics[width=0.65\columnwidth]{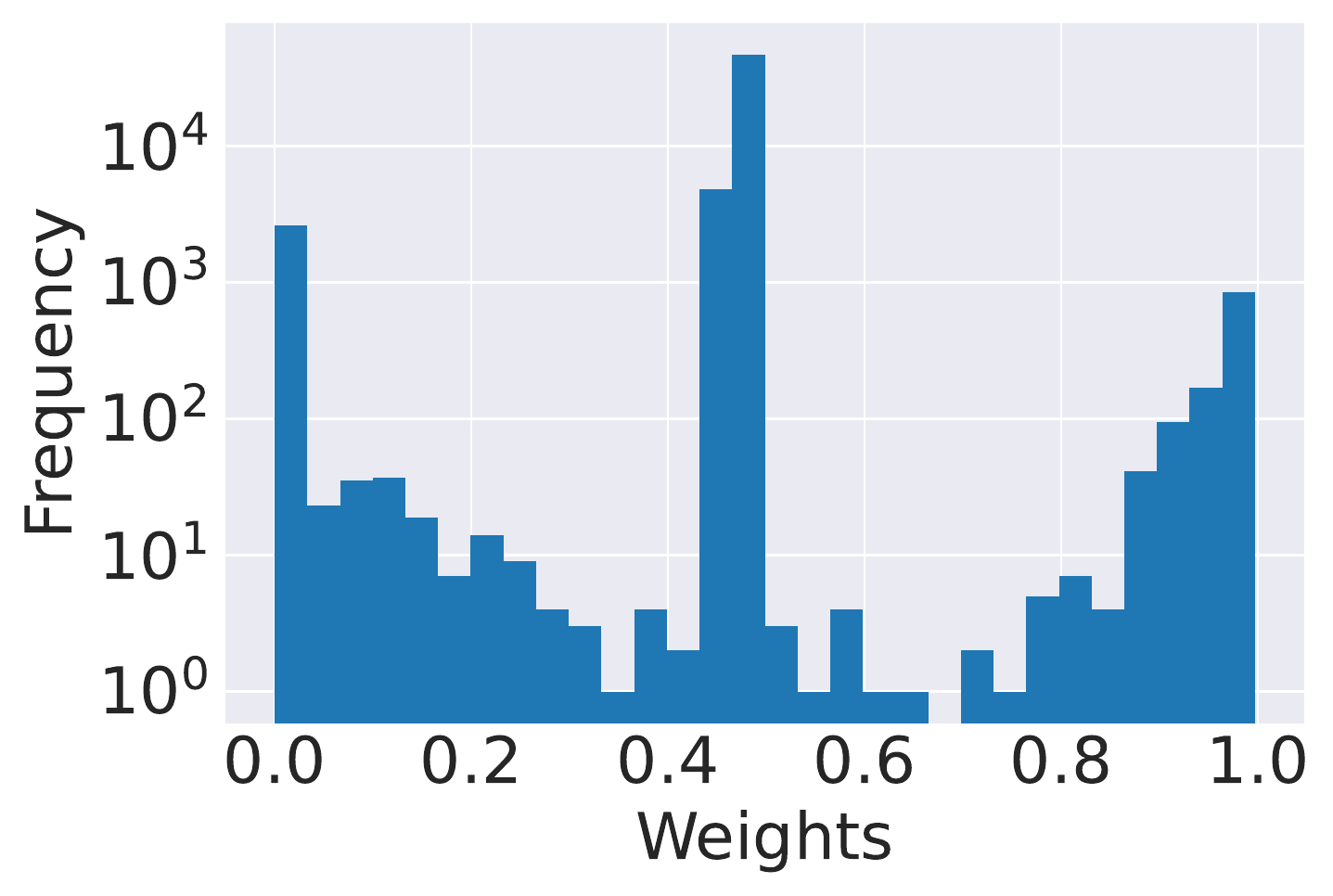}
	} \hfill
	\subfigure[{Scheme S3}]{
		\includegraphics[width=0.65\columnwidth]{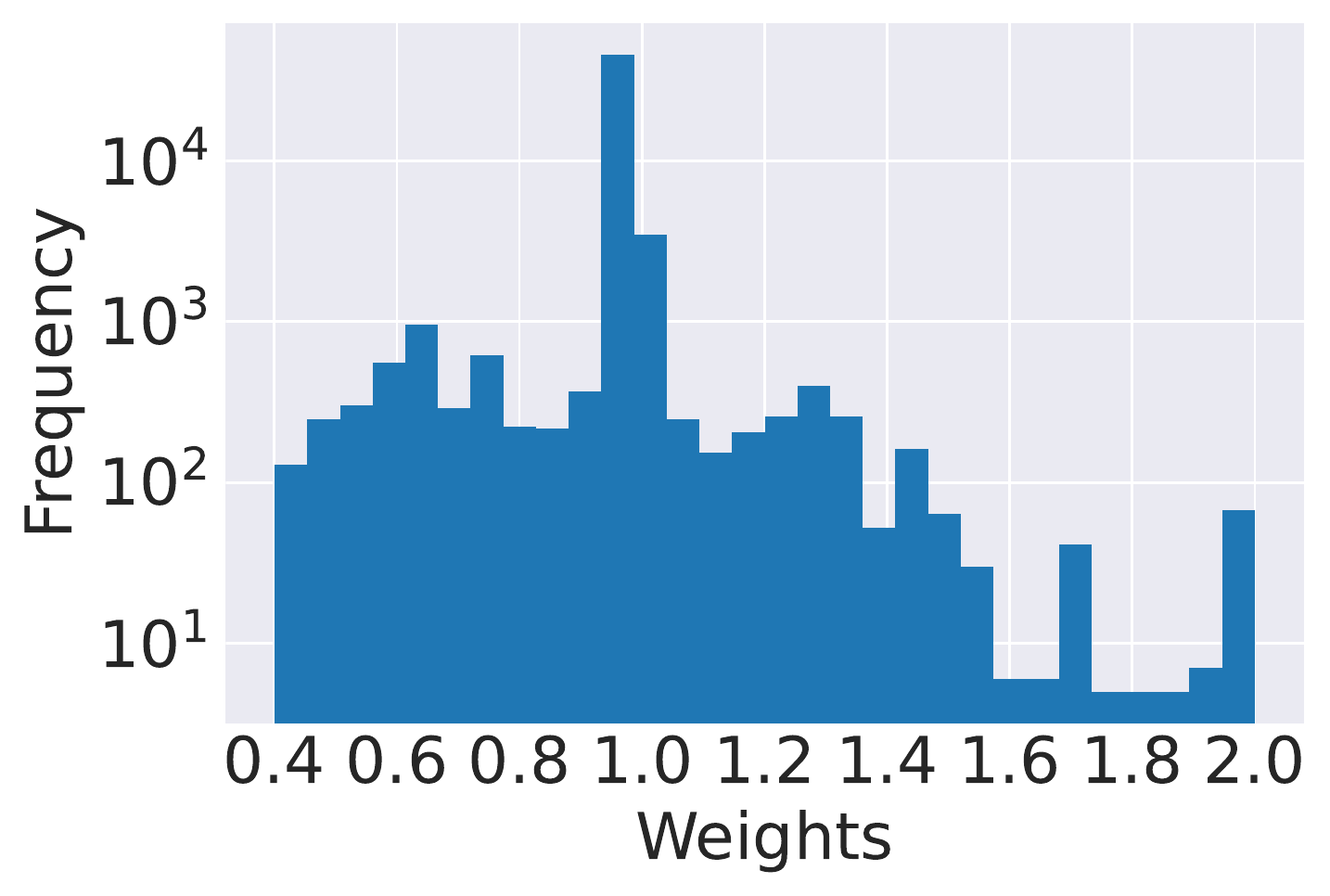}
	}
	\vspace*{-0.2cm}
	\caption{The distribution of learned weights in the three schemes. For scheme S2, we include the average weight of each slot in (a). For scheme S3, we illustrate the distribution of the sum of its two weighting parameters. }
	\label{fig:stardist}
% 	\vspace*{-0.15cm}
\end{figure*}

%  it involves.

On both MultiWOZ 2.4 and MultiWOZ 2.0*, we find that the performance of scheme S1 slightly lags behind ASSIST (with the best value of $\alpha$) in terms of joint goal accuracy, even though the weighting parameters learned in scheme S1 are slot-wise. The reason we speculate is that the learning algorithm fails to find the optimal slot-wise weighting parameters, but only the suboptimal ones. In \S\ref{sec:s1prior}, we show that scheme S1 can actually outperform ASSIST when the weighting parameters are initialized with the best value of $\alpha$ used in ASSIST.

As for scheme S3, Table~\ref{tab:main} shows that it achieves the best performance when AUX-DST is adopted as the primary model. For SOM-DST and STAR, its performance is comparable to the best results of ASSIST. Table~\ref{tab:main} also demonstrates that scheme S3 consistently outperforms scheme S1. However, it is inferior to scheme S2 when taking SOM-DST and STAR as the primary model. Recall that scheme S3 has the highest degree of flexibility in weighting parameters. These results suggest that while higher flexibility can in principle yield better results, the practical performance may not be particularly good due to the difficulty of learning optimal values for the weighting parameters.

% \usepackage{multirow}
\begin{table}[t]
\centering
\setlength{\tabcolsep}{1.5mm}
\begin{tabular}{l|cc|c}
\hline
\multirow{2}{*}{\textbf{Domain}} & \multicolumn{2}{c|}{\textbf{ASSIST}} & \textbf{MetaASSIST} \\ \cline{2-4} 
 & $\alpha=0.0$ & $\alpha=0.4$ & S2 \\ \hline
Attraction & 83.22 & 86.56 & \textbf{88.62}  \\
Hotel & 64.52 & 73.52 & \textbf{75.93} \\
Restaurant & 77.67 & 83.57 & \textbf{85.60} \\
Taxi & 54.76 & 63.65 & \textbf{67.71} \\
Train & 82.73 & \textbf{88.73} & 88.19 \\ \hline
\end{tabular}
\caption{Domain-specific JGA (\%) of SOM-DST on the test set of MultiWOZ 2.4.}
\label{tab:domain}
\end{table}

From Table~\ref{tab:main}, it can be further seen that scheme S1 consistently achieves higher validation performance than schemes S2 and S3 (except the joint turn accuracy of STAR). This might be confusing because scheme S1 underperforms schemes S2 and S3 on the test set. Moreover, the validation set is utilized to train the learnable functions in the three schemes. Hence, high validation performance is expected. However, we found that the distributions of the validation and test sets are not exactly the same (e.g., some slot values only appear in the test set). This implies that scheme S1 tends to overfit the validation data. While scheme S2 and scheme S3 suffer less from this issue, because the weighting parameters in them are not only related to state labels but also to the dialogue context.

% need to tune parameter per dataset , intance-wise

\subsection{Domain-Specific Accuracy}

Apart from the overall performance comparison, we also investigate the performance improvements in each domain. For this purpose, we report the domain-specific joint goal accuracy of SOM-DST on MultiWOZ 2.4 in Table~\ref{tab:domain}\footnote{After preprocessing, there are five domains left. Please refer to Appendix~\ref{sec:appdata} for more details.}. As can be observed, MetaASSIST achieves the best performance in four domains. In particular, MetaASSIST outperforms ASSIST ($\alpha=0.4$) by $4.06$ absolute points in the taxi domain. It can also be observed that MetaASSIST consistently outperforms ASSIST across all domains when ASSIST only considers vanilla labels ($\alpha=0.0$).

\subsection{Distribution of Learned Weights}

% To better understand how our proposed schemes differ from using a common weighting parameter, we illustrate the distribution of the learned weights in each scheme in Figure~\ref{fig:stardist}. 

Figure~\ref{fig:stardist} illustrates the distribution of the learned weights in each scheme. We conduct this study on MultiWOZ 2.4 and use STAR as the primary model. As shown in Figure~\ref{fig:stardist} (a), the learned weights in scheme S1 indeed vary across slots. For most slots, the weights are less than 0.5, indicating that their vanilla labels are of higher quality than pseudo labels. We also observe that the average weight of each slot in scheme S2 is smaller than the corresponding weight in scheme S1. In fact, the learned weights in scheme S2 are more consistent with the optimal value used in ASSIST (i.e., 0.4).

\begin{figure*}[!t]
   \begin{minipage}[t]{0.28\textwidth}
     \centering
     \includegraphics[width=0.88\linewidth]{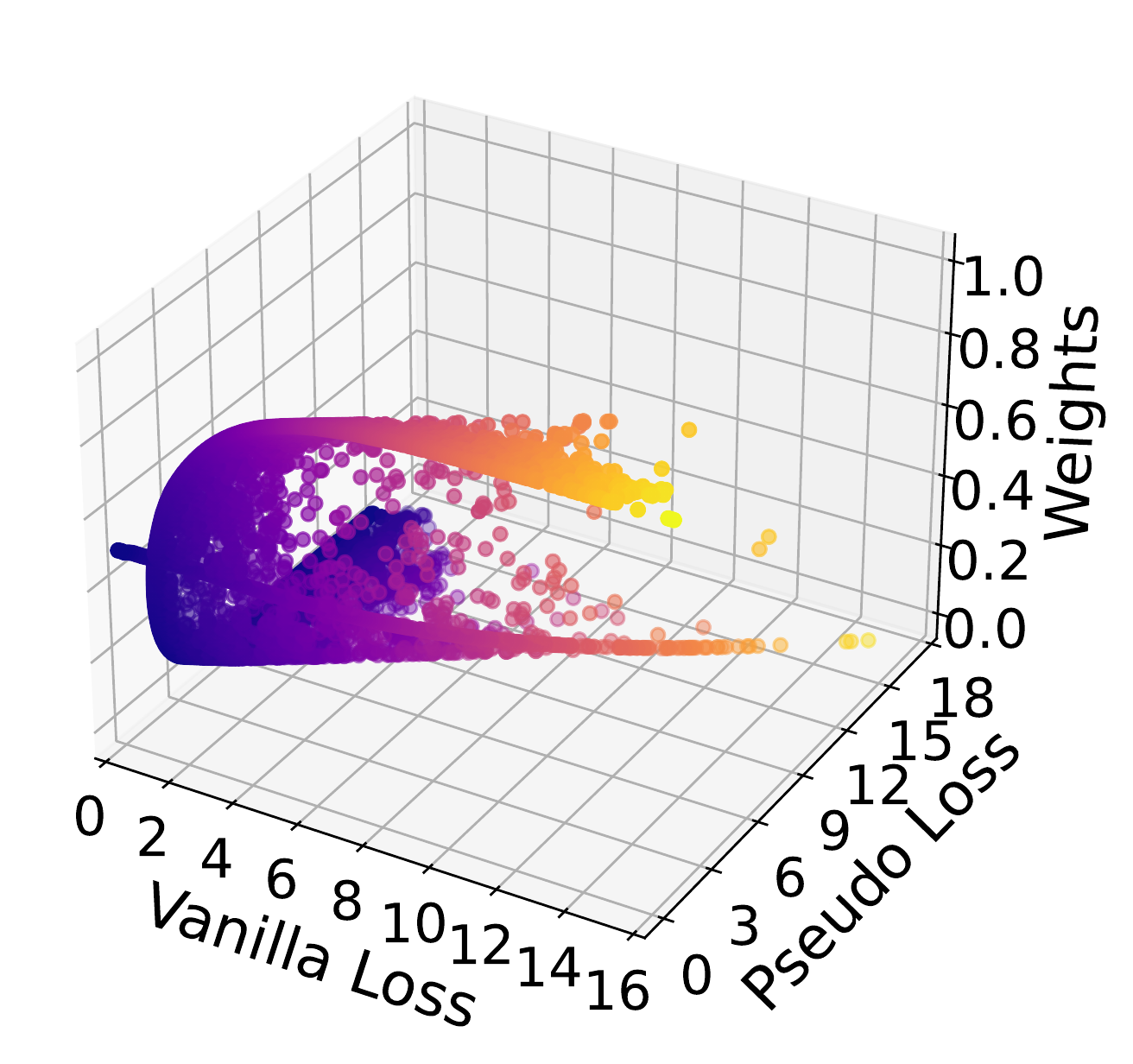}
    %  \vspace*{-0.25cm}
     \caption{The distribution of weights relative to loss values.}\label{fig:alpha}
   \end{minipage}\hfill
   \begin{minipage}[t]{0.32\textwidth}
     \centering
     \includegraphics[width=1.0\linewidth]{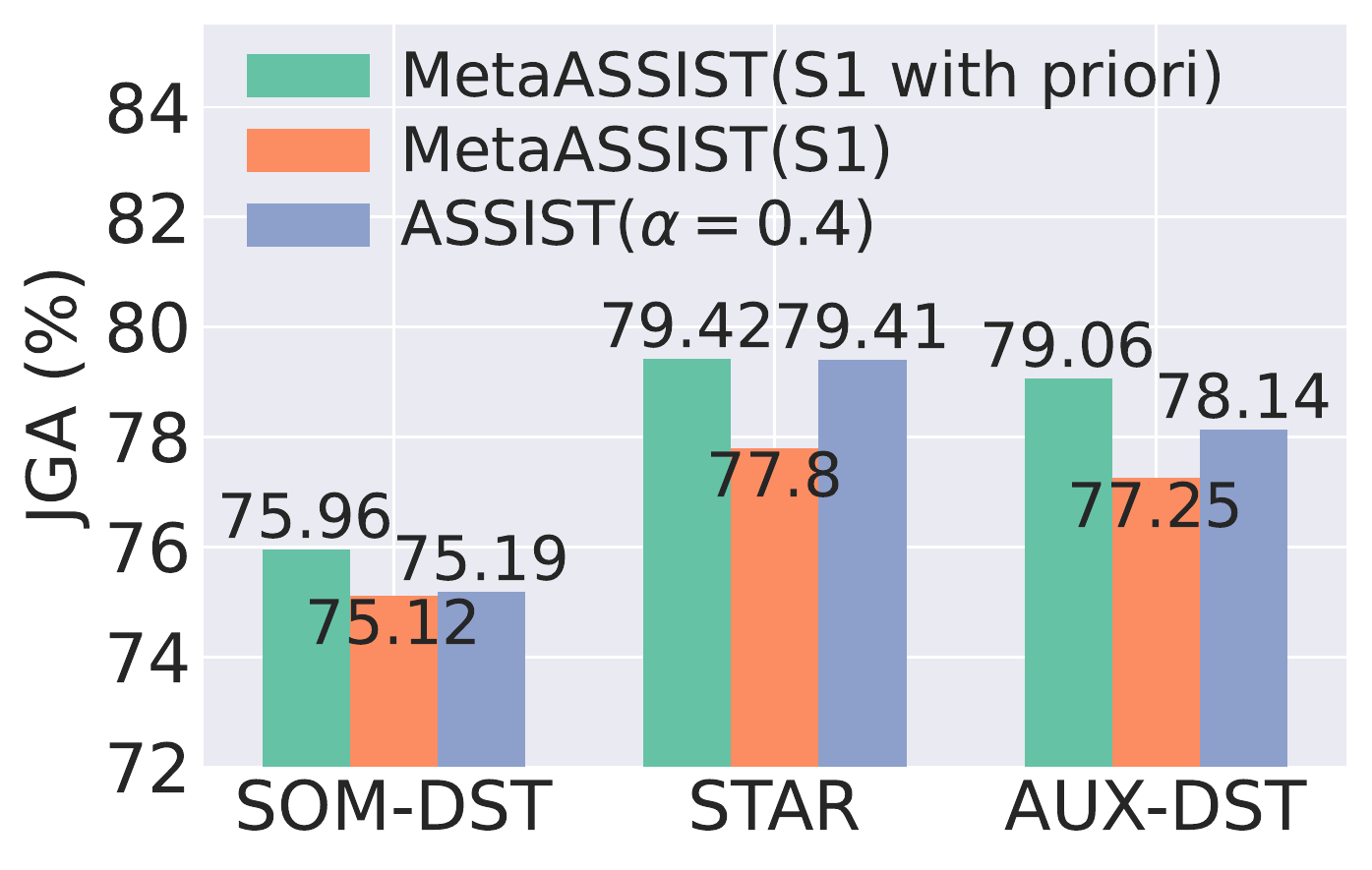}
    %  \vspace*{-0.55cm}
     \caption{Performance of scheme S1 with prior knowledge.}\label{fig:sizeofcleanset}
   \end{minipage}\hfill
   \begin{minipage}[t]{0.32\textwidth}
     \centering
     \includegraphics[width=0.928\linewidth]{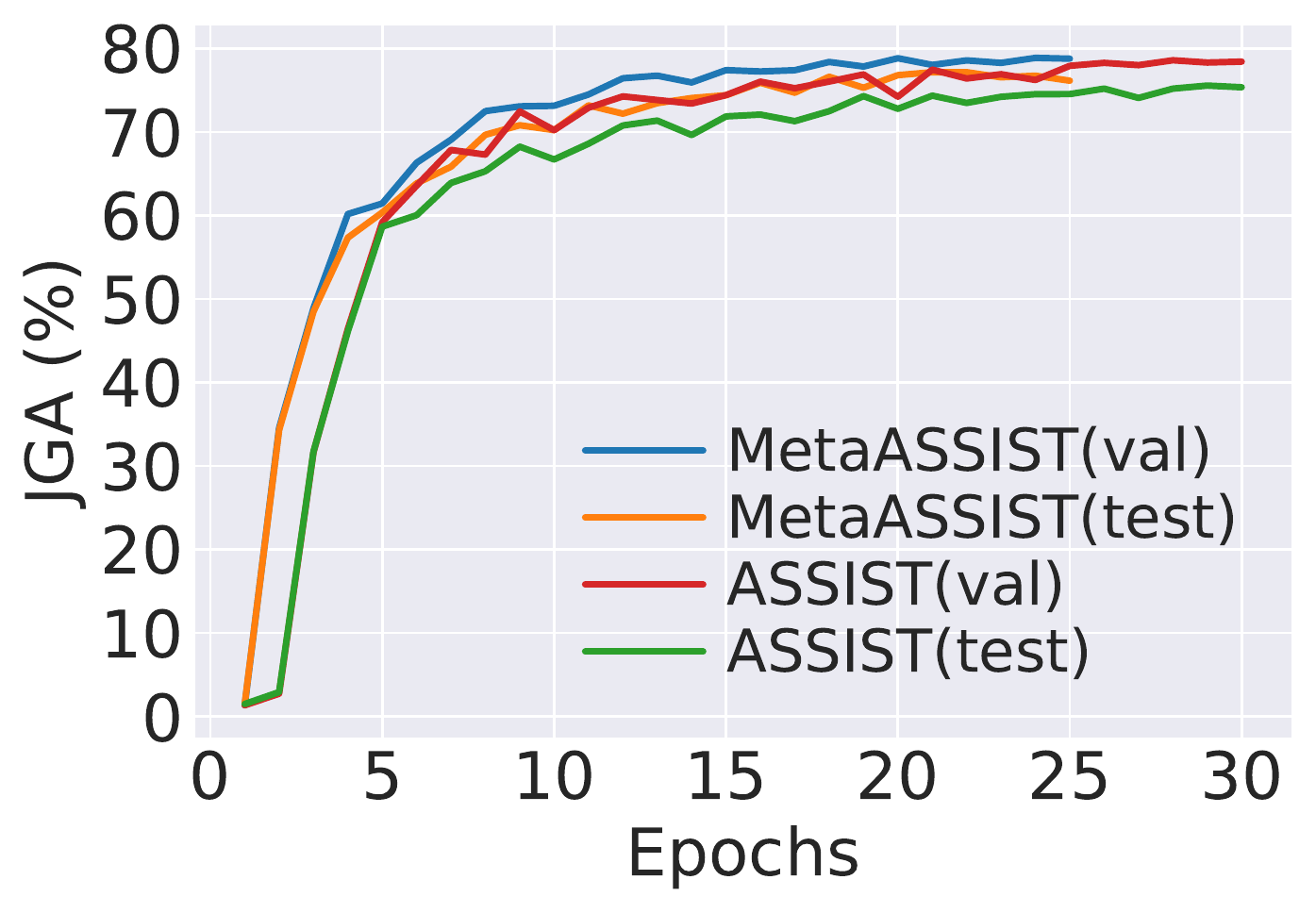}
    %  \vspace*{-0.25cm}
     \caption{The validation and test accuracy changing curves.}\label{fig:quan}
   \end{minipage}
%   \vspace*{-0.25cm}
\end{figure*}

Since schemes S2 and S3 are instance-wise, we randomly select a slot and plot the distribution of weights of this slot over all training samples. The results are shown in Figures~\ref{fig:stardist} (b) and (c). As can be seen, the learned weights vary across training samples. In scheme S2, although the learned weights for most training samples fall between 0.4 and 0.5, there are also many samples whose weights can be as small as 0 or as large as 1. Note that scheme S3 has two weighting parameters. We plot the distribution of their sums. It is interesting to observe that the sums of most training samples are around 1, even though we have removed the summation constraint in scheme S3. Nonetheless, we also observe that the sums of many training samples are less than 1, meaning that small weights have been assigned to both pseudo labels and vanilla labels.

Figure~\ref{fig:alpha} illustrates the distribution of weights in scheme S2 relative to loss values of pseudo labels and vanilla labels. We see that when both vanilla loss and pseudo loss are very small, the weights are around 0.5. When the vanilla loss is much smaller than the pseudo loss, the weights tend to be small. And when the pseudo loss is much smaller than the vanilla loss, the weights tend to be large.

The observations above confirm the strong capability of MetaASSIST in learning proper slot-wise (and instance-wise) weights based on loss values.

% The average weight of all slots are $0.4937$ and $0.4634$

\subsection{Scheme S1 with Prior Knowledge}
\label{sec:s1prior}

Given that the weighting parameters in scheme S1 are only slot-wise, we can readily initialize these parameters with any specified value. This implies that we can integrate prior knowledge into scheme S1. Specifically, we study using the optimal value of $\alpha$ found in ASSIST to initialize its weighting parameters. The results on MultiWOZ 2.4 are shown in Figure~\ref{fig:sizeofcleanset}. We observe that the prior knowledge can effectively improve the performance of scheme S1 for all three primary models. Furthermore, the results demonstrate that scheme S1 can outperform ASSIST when they use the same prior knowledge.

\subsection{Performance over Training Epochs}

% We investigate how the performance changes over training epochs. 

Figure~\ref{fig:quan} depicts the changing curves of validation accuracy and test accuracy over training epochs. We utilize SOM-DST as the primary model and conduct this experiment on MultiWOZ 2.4. For MetaASSIST, scheme S2 is applied. It is shown that the validation and test accuracy using MetaASSIST improves much faster than using ASSIST during early training epochs. In subsequent epochs, the validation accuracy with MetaASSIST is also higher and changes more smoothly. 

% \begin{figure}
%     \centering    \includegraphics[width=0.75\linewidth]{images/S1-with-priori3.pdf}
%     \caption{Caption}
%     \label{fig:my_label}
% \end{figure}

% \begin{figure}
%     \centering    \includegraphics[width=0.8\linewidth]{images/acc-to-epoch.pdf}
%     \caption{Caption}
%     \label{fig:my_label}
% \end{figure}

% \begin{figure}
%     \centering    \includegraphics[width=0.65\linewidth]{images/loss-to-weight.pdf}
%     \caption{Caption}
%     \label{fig:my_label}
% \end{figure}

% \begin{figure}
%     \centering    \includegraphics[width=0.8\linewidth]{images/dist-s2-v2.pdf}
%     \caption{Weight distribution of scheme S2}
%     \label{fig:my_label}
% \end{figure}

% \begin{figure}
%     \centering    \includegraphics[width=0.8\linewidth]{images/ave-dist-s2.pdf}
%     \caption{Weight distribution of scheme S2}
%     \label{fig:my_label}
% \end{figure}

\section{Related Work}

DST has been studied for more than one decade. Traditional DST models rely on a separate language understanding module to extract relevant information \citep{wang-lemon-2013-simple, williams-2014-web}. In recent years, designing DST models based on neural networks, especially pretrained language models, has become the mainstream and a large number of neural DST models have been proposed~\citep{mrksic-etal-2017-neural, wu-etal-2019-transferable, hosseini2020simple, zhu-etal-2020-efficient, lin-etal-2021-leveraging, lee-etal-2021-dialogue, zhao-etal-2021-effective-sequence,  feng-etal-2022-dynamic, hu2022context, guo-etal-2022-beyond, heck2022robust, manotumruksa2022similarity, sun2022tracking}.

Although these neural DST models have demonstrated good performance, they fail to consider the effect of label noise. It has been shown that no matter how noisy the training data are, neural models can easily overfit the training data \citep{zhang2021understanding}. As a result, the generalization performance of models trained on noisy data is usually unsatisfactory. Recently, \citet{ye-etal-2022-assist} proposed a general framework, ASSIST, to robustly train DST models on noisy data. Their experimental results show that several existing DST models can achieve much higher performance when trained under this framework. However, as discussed earlier, ASSIST contains a parameter that needs to be tuned on each dataset. Besides, the parameter is shared among all slots and all training samples, which we have shown to be suboptimal. Our proposed framework leverages meta learning to automatically learn slot-wise (and instance-wise) parameters, overcoming the limitations of ASSIST.

The essence of meta learning is learning to learn \citep{hospedales2021meta, zhu2022meta}, which makes it a  natural fit for our task of automatically learning parameters. We found that several existing works \cite{huang-etal-2020-meta, zeng2021domain, dingliwal-etal-2021-shot} have already applied meta learning to DST. These works directly adopt the MAML \citep{finn2017model} algorithm or its variants and focus on improving the few-shot learning ability of DST models. While our focus is to improve the robustness of DST models.

\section{Conclusion}

In this work, we proposed a meta learning-based general framework MetaASSIST to robustly train DST models on noisy data. MetaASSIST improves ASSIST by automatically learning slot-wise (and instance-wise) weighting parameters that are used to combine pseudo labels and vanilla labels. Our comprehensive experiments demonstrate the effectiveness of MetaASSIST. For future work, we plan to extend the current framework to utilize pseudo labels generated by multiple auxiliary models. 

% Given that the advantage of our proposed framework is to learn the best weighting parameters 

\section*{Limitations}

Our proposed framework MetaASSIST learns to weight pseudo labels and vanilla labels by minimizing the validation loss. Although it reduces the impact of label noise on model training, it runs the risk of biasing the trained model towards overfitting the validation data. One may argue that selecting the best model checkpoint based on validation performance is a standard strategy in machine learning. However, our empirical study shows that high performance on the validation set does not necessarily lead to high performance on the test set. This is because the validation set and test set are usually small, and their empirical data distributions can differ a lot. For a model trained with our framework to have high generalization performance, the validation set should be unbiased, but this requirement seems to be very demanding. In practice, we can augment the validation set to alleviate this problem.

Another limitation is that our proposed learning algorithm is more time-consuming than regular model training. As described in Algorithm~\ref{alg:cap}, for each training batch, the model needs to perform two forward and backward passes (the first pass to obtain the interim model, the second pass to obtain the updated model). For each validation (meta) batch, the model also needs to perform one forward and backwad pass. Therefore, the learning algorithm needs $3\times$ training time compared to regular training. Nonetheless, compared to ASSIST, the proposed framework MetaASSIST is more time-efficient. For ASSIST, we need to try a large number of values for $\alpha$ to find the best one.

% Last but not least, it is worth emphasizing that the effectiveness of our proposed framework can be greatly affected by the optimizer used. We realize that it is challenging to find the optimal weighting parameters using existing popular optimizers. 

\section*{Ethics Statement}

The DST module is an essential component in many industrial and commercial dialogue systems. Performance improvements on DST can help these systems better understand users' requirements, thereby improving user satisfaction. Our proposed framework could be applied to these systems and improve their DST performance. The proposed framework can also be applied to other NLP and machine learning applications.

\section*{Acknowledgements}

This work was funded by the Alan Turing Institute under the EPSRC grant EP/N510129/1 and the EPSRC Fellowship titled “Task Based Information Retrieval” and grant reference number EP/P024289/1.

% \section*{Acknowledgements}

\bibliography{anthology,custom}

\begin{thebibliography}{39}
\expandafter\ifx\csname natexlab\endcsname\relax\def\natexlab#1{#1}\fi

\bibitem[{Balaraman et~al.(2021)Balaraman, Sheikhalishahi, and
  Magnini}]{balaraman-etal-2021-recent}
Vevake Balaraman, Seyedmostafa Sheikhalishahi, and Bernardo Magnini. 2021.
\newblock \href {https://aclanthology.org/2021.sigdial-1.25} {Recent neural
  methods on dialogue state tracking for task-oriented dialogue systems: A
  survey}.
\newblock In \emph{Proceedings of the 22nd Annual Meeting of the Special
  Interest Group on Discourse and Dialogue}, pages 239--251, Singapore and
  Online. Association for Computational Linguistics.

\bibitem[{Budzianowski et~al.(2018)Budzianowski, Wen, Tseng, Casanueva, Ultes,
  Ramadan, and Ga{\v{s}}i{\'c}}]{budzianowski-etal-2018-multiwoz}
Pawe{\l} Budzianowski, Tsung-Hsien Wen, Bo-Hsiang Tseng, I{\~n}igo Casanueva,
  Stefan Ultes, Osman Ramadan, and Milica Ga{\v{s}}i{\'c}. 2018.
\newblock \href {https://doi.org/10.18653/v1/D18-1547} {{M}ulti{WOZ} - a
  large-scale multi-domain {W}izard-of-{O}z dataset for task-oriented dialogue
  modelling}.
\newblock In \emph{Proceedings of the 2018 Conference on Empirical Methods in
  Natural Language Processing}, pages 5016--5026, Brussels, Belgium.
  Association for Computational Linguistics.

\bibitem[{Campagna et~al.(2020)Campagna, Foryciarz, Moradshahi, and
  Lam}]{campagna-etal-2020-zero}
Giovanni Campagna, Agata Foryciarz, Mehrad Moradshahi, and Monica Lam. 2020.
\newblock \href {https://doi.org/10.18653/v1/2020.acl-main.12} {Zero-shot
  transfer learning with synthesized data for multi-domain dialogue state
  tracking}.
\newblock In \emph{Proceedings of the 58th Annual Meeting of the Association
  for Computational Linguistics}, pages 122--132, Online. Association for
  Computational Linguistics.

\bibitem[{Devlin et~al.(2019)Devlin, Chang, Lee, and
  Toutanova}]{devlin-etal-2019-bert}
Jacob Devlin, Ming-Wei Chang, Kenton Lee, and Kristina Toutanova. 2019.
\newblock \href {https://doi.org/10.18653/v1/N19-1423} {{BERT}: Pre-training of
  deep bidirectional transformers for language understanding}.
\newblock In \emph{Proceedings of the 2019 Conference of the North {A}merican
  Chapter of the Association for Computational Linguistics: Human Language
  Technologies, Volume 1 (Long and Short Papers)}, pages 4171--4186,
  Minneapolis, Minnesota. Association for Computational Linguistics.

\bibitem[{Dingliwal et~al.(2021)Dingliwal, Gao, Agarwal, Lin, Chung, and
  Hakkani-Tur}]{dingliwal-etal-2021-shot}
Saket Dingliwal, Shuyang Gao, Sanchit Agarwal, Chien-Wei Lin, Tagyoung Chung,
  and Dilek Hakkani-Tur. 2021.
\newblock \href {https://doi.org/10.18653/v1/2021.eacl-main.148} {Few shot
  dialogue state tracking using meta-learning}.
\newblock In \emph{Proceedings of the 16th Conference of the European Chapter
  of the Association for Computational Linguistics: Main Volume}, pages
  1730--1739, Online. Association for Computational Linguistics.

\bibitem[{Eric et~al.(2020)Eric, Goel, Paul, Sethi, Agarwal, Gao, Kumar, Goyal,
  Ku, and Hakkani-Tur}]{eric-etal-2020-multiwoz}
Mihail Eric, Rahul Goel, Shachi Paul, Abhishek Sethi, Sanchit Agarwal, Shuyang
  Gao, Adarsh Kumar, Anuj Goyal, Peter Ku, and Dilek Hakkani-Tur. 2020.
\newblock \href {https://aclanthology.org/2020.lrec-1.53} {{M}ulti{WOZ} 2.1: A
  consolidated multi-domain dialogue dataset with state corrections and state
  tracking baselines}.
\newblock In \emph{Proceedings of the 12th Language Resources and Evaluation
  Conference}, pages 422--428, Marseille, France. European Language Resources
  Association.

\bibitem[{Feng et~al.(2022)Feng, Lipani, Ye, Zhang, and
  Yilmaz}]{feng-etal-2022-dynamic}
Yue Feng, Aldo Lipani, Fanghua Ye, Qiang Zhang, and Emine Yilmaz. 2022.
\newblock \href {https://doi.org/10.18653/v1/2022.acl-long.10} {Dynamic schema
  graph fusion network for multi-domain dialogue state tracking}.
\newblock In \emph{Proceedings of the 60th Annual Meeting of the Association
  for Computational Linguistics (Volume 1: Long Papers)}, pages 115--126,
  Dublin, Ireland. Association for Computational Linguistics.

\bibitem[{Finn et~al.(2017)Finn, Abbeel, and Levine}]{finn2017model}
Chelsea Finn, Pieter Abbeel, and Sergey Levine. 2017.
\newblock Model-agnostic meta-learning for fast adaptation of deep networks.
\newblock In \emph{International conference on machine learning}, pages
  1126--1135. PMLR.

\bibitem[{Guo et~al.(2022)Guo, Shuang, Li, Wang, and
  Liu}]{guo-etal-2022-beyond}
Jinyu Guo, Kai Shuang, Jijie Li, Zihan Wang, and Yixuan Liu. 2022.
\newblock \href {https://doi.org/10.18653/v1/2022.acl-long.165} {Beyond the
  granularity: Multi-perspective dialogue collaborative selection for dialogue
  state tracking}.
\newblock In \emph{Proceedings of the 60th Annual Meeting of the Association
  for Computational Linguistics (Volume 1: Long Papers)}, pages 2320--2332,
  Dublin, Ireland. Association for Computational Linguistics.

\bibitem[{Heck et~al.(2022)Heck, Lubis, van Niekerk, Feng, Geishauser, Lin, and
  Ga{\v{s}}i{\'c}}]{heck2022robust}
Michael Heck, Nurul Lubis, Carel van Niekerk, Shutong Feng, Christian
  Geishauser, Hsien-Chin Lin, and Milica Ga{\v{s}}i{\'c}. 2022.
\newblock Robust dialogue state tracking with weak supervision and sparse data.
\newblock \emph{arXiv preprint arXiv:2202.03354}.

\bibitem[{Henderson et~al.(2014)Henderson, Thomson, and
  Williams}]{henderson-etal-2014-second}
Matthew Henderson, Blaise Thomson, and Jason~D. Williams. 2014.
\newblock \href {https://doi.org/10.3115/v1/W14-4337} {The second dialog state
  tracking challenge}.
\newblock In \emph{Proceedings of the 15th Annual Meeting of the Special
  Interest Group on Discourse and Dialogue ({SIGDIAL})}, pages 263--272,
  Philadelphia, PA, U.S.A. Association for Computational Linguistics.

\bibitem[{Hospedales et~al.(2021)Hospedales, Antoniou, Micaelli, and
  Storkey}]{hospedales2021meta}
Timothy~M Hospedales, Antreas Antoniou, Paul Micaelli, and Amos~J Storkey.
  2021.
\newblock Meta-learning in neural networks: A survey.
\newblock \emph{IEEE Transactions on Pattern Analysis and Machine
  Intelligence}.

\bibitem[{Hosseini-Asl et~al.(2020)Hosseini-Asl, McCann, Wu, Yavuz, and
  Socher}]{hosseini2020simple}
Ehsan Hosseini-Asl, Bryan McCann, Chien-Sheng Wu, Semih Yavuz, and Richard
  Socher. 2020.
\newblock A simple language model for task-oriented dialogue.
\newblock \emph{Advances in Neural Information Processing Systems},
  33:20179--20191.

\bibitem[{Hu et~al.(2022)Hu, Lee, Xie, Yu, Smith, and
  Ostendorf}]{hu2022context}
Yushi Hu, Chia-Hsuan Lee, Tianbao Xie, Tao Yu, Noah~A Smith, and Mari
  Ostendorf. 2022.
\newblock In-context learning for few-shot dialogue state tracking.
\newblock \emph{arXiv preprint arXiv:2203.08568}.

\bibitem[{Huang et~al.(2020)Huang, Feng, Hu, Wu, Du, and
  Ma}]{huang-etal-2020-meta}
Yi~Huang, Junlan Feng, Min Hu, Xiaoting Wu, Xiaoyu Du, and Shuo Ma. 2020.
\newblock \href {https://doi.org/10.18653/v1/2020.acl-main.636}
  {Meta-reinforced multi-domain state generator for dialogue systems}.
\newblock In \emph{Proceedings of the 58th Annual Meeting of the Association
  for Computational Linguistics}, pages 7109--7118, Online. Association for
  Computational Linguistics.

\bibitem[{Kim et~al.(2020)Kim, Yang, Kim, and Lee}]{kim-etal-2020-efficient}
Sungdong Kim, Sohee Yang, Gyuwan Kim, and Sang-Woo Lee. 2020.
\newblock \href {https://doi.org/10.18653/v1/2020.acl-main.53} {Efficient
  dialogue state tracking by selectively overwriting memory}.
\newblock In \emph{Proceedings of the 58th Annual Meeting of the Association
  for Computational Linguistics}, pages 567--582, Online. Association for
  Computational Linguistics.

\bibitem[{Lee et~al.(2021)Lee, Cheng, and Ostendorf}]{lee-etal-2021-dialogue}
Chia-Hsuan Lee, Hao Cheng, and Mari Ostendorf. 2021.
\newblock \href {https://doi.org/10.18653/v1/2021.emnlp-main.404} {Dialogue
  state tracking with a language model using schema-driven prompting}.
\newblock In \emph{Proceedings of the 2021 Conference on Empirical Methods in
  Natural Language Processing}, pages 4937--4949, Online and Punta Cana,
  Dominican Republic. Association for Computational Linguistics.

\bibitem[{Lin et~al.(2021)Lin, Liu, Moon, Crook, Zhou, Wang, Yu, Madotto, Cho,
  and Subba}]{lin-etal-2021-leveraging}
Zhaojiang Lin, Bing Liu, Seungwhan Moon, Paul Crook, Zhenpeng Zhou, Zhiguang
  Wang, Zhou Yu, Andrea Madotto, Eunjoon Cho, and Rajen Subba. 2021.
\newblock \href {https://doi.org/10.18653/v1/2021.naacl-main.448} {Leveraging
  slot descriptions for zero-shot cross-domain dialogue {S}tate{T}racking}.
\newblock In \emph{Proceedings of the 2021 Conference of the North American
  Chapter of the Association for Computational Linguistics: Human Language
  Technologies}, pages 5640--5648, Online. Association for Computational
  Linguistics.

\bibitem[{Loshchilov and Hutter(2017)}]{loshchilov2017decoupled}
Ilya Loshchilov and Frank Hutter. 2017.
\newblock Decoupled weight decay regularization.
\newblock \emph{arXiv preprint arXiv:1711.05101}.

\bibitem[{Manotumruksa et~al.(2021)Manotumruksa, Dalton, Meij, and
  Yilmaz}]{manotumruksa-etal-2021-improving-dialogue}
Jarana Manotumruksa, Jeff Dalton, Edgar Meij, and Emine Yilmaz. 2021.
\newblock \href {https://doi.org/10.18653/v1/2021.findings-emnlp.144}
  {Improving dialogue state tracking with turn-based loss function and
  sequential data augmentation}.
\newblock In \emph{Findings of the Association for Computational Linguistics:
  EMNLP 2021}, pages 1674--1683, Punta Cana, Dominican Republic. Association
  for Computational Linguistics.

\bibitem[{Manotumruksa et~al.(2022)Manotumruksa, Dalton, Meij, and
  Yilmaz}]{manotumruksa2022similarity}
Jarana Manotumruksa, Jeffrey Dalton, Edgar Meij, and Emine Yilmaz. 2022.
\newblock Similarity-based multi-domain dialogue state tracking with copy
  mechanisms for task-based virtual personal assistants.
\newblock In \emph{Proceedings of the ACM Web Conference 2022}, pages
  2006--2014.

\bibitem[{Mrk{\v{s}}i{\'c} et~al.(2017)Mrk{\v{s}}i{\'c}, {\'O}~S{\'e}aghdha,
  Wen, Thomson, and Young}]{mrksic-etal-2017-neural}
Nikola Mrk{\v{s}}i{\'c}, Diarmuid {\'O}~S{\'e}aghdha, Tsung-Hsien Wen, Blaise
  Thomson, and Steve Young. 2017.
\newblock \href {https://doi.org/10.18653/v1/P17-1163} {Neural belief tracker:
  Data-driven dialogue state tracking}.
\newblock In \emph{Proceedings of the 55th Annual Meeting of the Association
  for Computational Linguistics (Volume 1: Long Papers)}, pages 1777--1788,
  Vancouver, Canada. Association for Computational Linguistics.

\bibitem[{Rastogi et~al.(2020)Rastogi, Zang, Sunkara, Gupta, and
  Khaitan}]{rastogi2020towards}
Abhinav Rastogi, Xiaoxue Zang, Srinivas Sunkara, Raghav Gupta, and Pranav
  Khaitan. 2020.
\newblock Towards scalable multi-domain conversational agents: The
  schema-guided dialogue dataset.
\newblock In \emph{Proceedings of the AAAI Conference on Artificial
  Intelligence}, volume~34, pages 8689--8696.

\bibitem[{Ren et~al.(2018)Ren, Zeng, Yang, and Urtasun}]{ren2018learning}
Mengye Ren, Wenyuan Zeng, Bin Yang, and Raquel Urtasun. 2018.
\newblock Learning to reweight examples for robust deep learning.
\newblock In \emph{International conference on machine learning}, pages
  4334--4343. PMLR.

\bibitem[{Rumelhart et~al.(1986)Rumelhart, Hinton, and
  Williams}]{rumelhart1986learning}
David~E Rumelhart, Geoffrey~E Hinton, and Ronald~J Williams. 1986.
\newblock Learning representations by back-propagating errors.
\newblock \emph{nature}, 323(6088):533--536.

\bibitem[{Shin et~al.(2022)Shin, Yu, Moon, Madotto, and
  Park}]{shin-etal-2022-dialogue}
Jamin Shin, Hangyeol Yu, Hyeongdon Moon, Andrea Madotto, and Juneyoung Park.
  2022.
\newblock \href {https://doi.org/10.18653/v1/2022.findings-acl.302} {Dialogue
  summaries as dialogue states ({DS}2), template-guided summarization for
  few-shot dialogue state tracking}.
\newblock In \emph{Findings of the Association for Computational Linguistics:
  ACL 2022}, pages 3824--3846, Dublin, Ireland. Association for Computational
  Linguistics.

\bibitem[{Sun et~al.(2022)Sun, Huang, and Ding}]{sun2022tracking}
Zhoujian Sun, Zhengxing Huang, and Nai Ding. 2022.
\newblock On tracking dialogue state by inheriting slot values in mentioned
  slot pools.
\newblock \emph{arXiv preprint arXiv:2202.07156}.

\bibitem[{Wang et~al.(2022)Wang, Zhao, Bao, Duan, Wu, and He}]{wang2022luna}
Yifan Wang, Jing Zhao, Junwei Bao, Chaoqun Duan, Youzheng Wu, and Xiaodong He.
  2022.
\newblock Luna: Learning slot-turn alignment for dialogue state tracking.
\newblock \emph{arXiv preprint arXiv:2205.02550}.

\bibitem[{Wang and Lemon(2013)}]{wang-lemon-2013-simple}
Zhuoran Wang and Oliver Lemon. 2013.
\newblock \href {https://aclanthology.org/W13-4067} {A simple and generic
  belief tracking mechanism for the dialog state tracking challenge: On the
  believability of observed information}.
\newblock In \emph{Proceedings of the {SIGDIAL} 2013 Conference}, pages
  423--432, Metz, France. Association for Computational Linguistics.

\bibitem[{Williams(2014)}]{williams-2014-web}
Jason~D. Williams. 2014.
\newblock \href {https://doi.org/10.3115/v1/W14-4339} {Web-style ranking and
  {SLU} combination for dialog state tracking}.
\newblock In \emph{Proceedings of the 15th Annual Meeting of the Special
  Interest Group on Discourse and Dialogue ({SIGDIAL})}, pages 282--291,
  Philadelphia, PA, U.S.A. Association for Computational Linguistics.

\bibitem[{Wu et~al.(2019)Wu, Madotto, Hosseini-Asl, Xiong, Socher, and
  Fung}]{wu-etal-2019-transferable}
Chien-Sheng Wu, Andrea Madotto, Ehsan Hosseini-Asl, Caiming Xiong, Richard
  Socher, and Pascale Fung. 2019.
\newblock \href {https://doi.org/10.18653/v1/P19-1078} {Transferable
  multi-domain state generator for task-oriented dialogue systems}.
\newblock In \emph{Proceedings of the 57th Annual Meeting of the Association
  for Computational Linguistics}, pages 808--819, Florence, Italy. Association
  for Computational Linguistics.

\bibitem[{Ye et~al.(2022)Ye, Feng, and Yilmaz}]{ye-etal-2022-assist}
Fanghua Ye, Yue Feng, and Emine Yilmaz. 2022.
\newblock \href {https://doi.org/10.18653/v1/2022.findings-acl.214} {{ASSIST}:
  Towards label noise-robust dialogue state tracking}.
\newblock In \emph{Findings of the Association for Computational Linguistics:
  ACL 2022}, pages 2719--2731, Dublin, Ireland. Association for Computational
  Linguistics.

\bibitem[{Ye et~al.(2021{\natexlab{a}})Ye, Manotumruksa, and
  Yilmaz}]{ye2021multiwoz}
Fanghua Ye, Jarana Manotumruksa, and Emine Yilmaz. 2021{\natexlab{a}}.
\newblock Multiwoz 2.4: A multi-domain task-oriented dialogue dataset with
  essential annotation corrections to improve state tracking evaluation.
\newblock \emph{arXiv preprint arXiv:2104.00773}.

\bibitem[{Ye et~al.(2021{\natexlab{b}})Ye, Manotumruksa, Zhang, Li, and
  Yilmaz}]{ye2021slot}
Fanghua Ye, Jarana Manotumruksa, Qiang Zhang, Shenghui Li, and Emine Yilmaz.
  2021{\natexlab{b}}.
\newblock Slot self-attentive dialogue state tracking.
\newblock In \emph{Proceedings of the Web Conference 2021}, pages 1598--1608.

\bibitem[{Zeng et~al.(2021)Zeng, Yin, Liu, Ge, and Su}]{zeng2021domain}
Jiali Zeng, Yongjing Yin, Yang Liu, Yubin Ge, and Jinsong Su. 2021.
\newblock Domain adaptive meta-learning for dialogue state tracking.
\newblock \emph{IEEE/ACM Transactions on Audio, Speech, and Language
  Processing}, 29:2493--2501.

\bibitem[{Zhang et~al.(2021)Zhang, Bengio, Hardt, Recht, and
  Vinyals}]{zhang2021understanding}
Chiyuan Zhang, Samy Bengio, Moritz Hardt, Benjamin Recht, and Oriol Vinyals.
  2021.
\newblock Understanding deep learning (still) requires rethinking
  generalization.
\newblock \emph{Communications of the ACM}, 64(3):107--115.

\bibitem[{Zhao et~al.(2021)Zhao, Mahdieh, Zhang, Cao, and
  Wu}]{zhao-etal-2021-effective-sequence}
Jeffrey Zhao, Mahdis Mahdieh, Ye~Zhang, Yuan Cao, and Yonghui Wu. 2021.
\newblock \href {https://doi.org/10.18653/v1/2021.emnlp-main.593} {Effective
  sequence-to-sequence dialogue state tracking}.
\newblock In \emph{Proceedings of the 2021 Conference on Empirical Methods in
  Natural Language Processing}, pages 7486--7493, Online and Punta Cana,
  Dominican Republic. Association for Computational Linguistics.

\bibitem[{Zhu et~al.(2022)Zhu, Shen, Hedderich, and Klakow}]{zhu2022meta}
Dawei Zhu, Xiaoyu Shen, Michael~A Hedderich, and Dietrich Klakow. 2022.
\newblock Meta self-refinement for robust learning with weak supervision.
\newblock \emph{arXiv preprint arXiv:2205.07290}.

\bibitem[{Zhu et~al.(2020)Zhu, Li, Chen, and Yu}]{zhu-etal-2020-efficient}
Su~Zhu, Jieyu Li, Lu~Chen, and Kai Yu. 2020.
\newblock \href {https://doi.org/10.18653/v1/2020.findings-emnlp.68} {Efficient
  context and schema fusion networks for multi-domain dialogue state tracking}.
\newblock In \emph{Findings of the Association for Computational Linguistics:
  EMNLP 2020}, pages 766--781, Online. Association for Computational
  Linguistics.

\end{thebibliography}
\bibliographystyle{acl_natbib}

\clearpage 
\appendix

\section{Implementation Details}
\label{sec:appdata}

The MultiWOZ dataset contains seven domains: \textit{attraction}, \textit{hotel}, \textit{restaurant}, \textit{taxi}, \textit{train}, \textit{hospital} and \textit{police}. However, the \textit{hospital} domain and \textit{police} domain only occur in the training set. Following previous works \citep{wu-etal-2019-transferable, kim-etal-2020-efficient}, we remove the two domains. This results in five domains with 30 slots in total.

For a fair comparison with ASSIST, we directly employ the pseudo labels published by the authors instead of training a new auxiliary model ourselves to generate pseudo labels. For all primary models, we modify their released code to implement our learning algorithm. All the primary models adopt BERT \citep{devlin-etal-2019-bert} as the dialogue context encoder and are initialized using the pretrained BERT-base-uncased model. As for the MLP network in schemes S2 and S3, we set the hidden layer dimension to 768. The output layer dimension is fixed at 1. The MLP network is randomly initialized. For scheme S1, we initialize the weighting parameters to be 0.5. For all primary models, we adopt their default hyperparameter settings, except the training epochs. For SOM-DST, we halve the batch size due to its high GPU memory requirement. We fix the validation (meta) batch size at 8 for all three primary models. AdamW \citep{loshchilov2017decoupled} is employed as the optimizer and a linear scheduler with warmup is created to adjust the learning rate dynamically. The warmup proportion is fixed at 0.1. Tables~\ref{tab:para} and \ref{tab:para2} summarize the training epochs and peak validation (meta) learning rate in each scheme for each model.

% \usepackage{multirow}
\begin{table}[t]
\centering
\setlength{\tabcolsep}{1.2mm}
\begin{tabular}{c|c|cc}
\hline
\textbf{Model} & \textbf{Scheme} & \textbf{Epochs} & \textbf{Learning Rate} \\ \hline
\multirow{3}{*}{SOM-DST} & S1 & 30 & 4e-5 \\
 & S2 & 25 & 2e-5 \\
 & S3 & 25 & 1e-5 \\ \hline
\multirow{3}{*}{STAR} & S1 & 15 & 5e-5 \\
 & S2 & 15 & 1e-5 \\
 & S3 & 12 & 3e-5 \\ \hline
\multirow{3}{*}{AUX-DST} & S1 & 12 & 1e-4 \\
 & S2 & 15 & 2.5e-5 \\
 & S3 & 12 & 2e-5 \\ \hline
\end{tabular}
\caption{Number of maximum training epochs and peak validation (meta) learning rate on MultiWOZ 2.4.}
\label{tab:para}
\end{table}

\begin{table}[t]
\centering
\setlength{\tabcolsep}{1.2mm}
\begin{tabular}{c|c|cc}
\hline
\textbf{Model} & \textbf{Scheme} & \textbf{Epochs} & \textbf{Learning Rate} \\ \hline
\multirow{3}{*}{SOM-DST} & S1 & 25 & 3e-5 \\
 & S2 & 25 & 1e-5 \\
 & S3 & 25 & 8e-6 \\ \hline
\end{tabular}
\caption{Number of maximum training epochs and peak validation (meta) learning rate on MultiWOZ 2.0*.}
\label{tab:para2}
\end{table}

\section{Convergence Analysis}

Considering that our proposed learning algorithm optimizes the primary model and the learnable function alternately, it is meaningful to study its convergence. To this end, we plot the loss value curves of training batch and validation (meta) batch over training steps. We adopt AUX-DST as the primary model and apply scheme S1 to learn the weighting parameters. We conduct this experiment on MultiWOZ 2.4. The results are illustrated in Figure~\ref{fig:conv}. As can be observed, the training loss and validation (meta) loss both converge to relatively small values after sufficient training steps.

\begin{figure} [t]
    \centering    \includegraphics[width=1.0\linewidth]{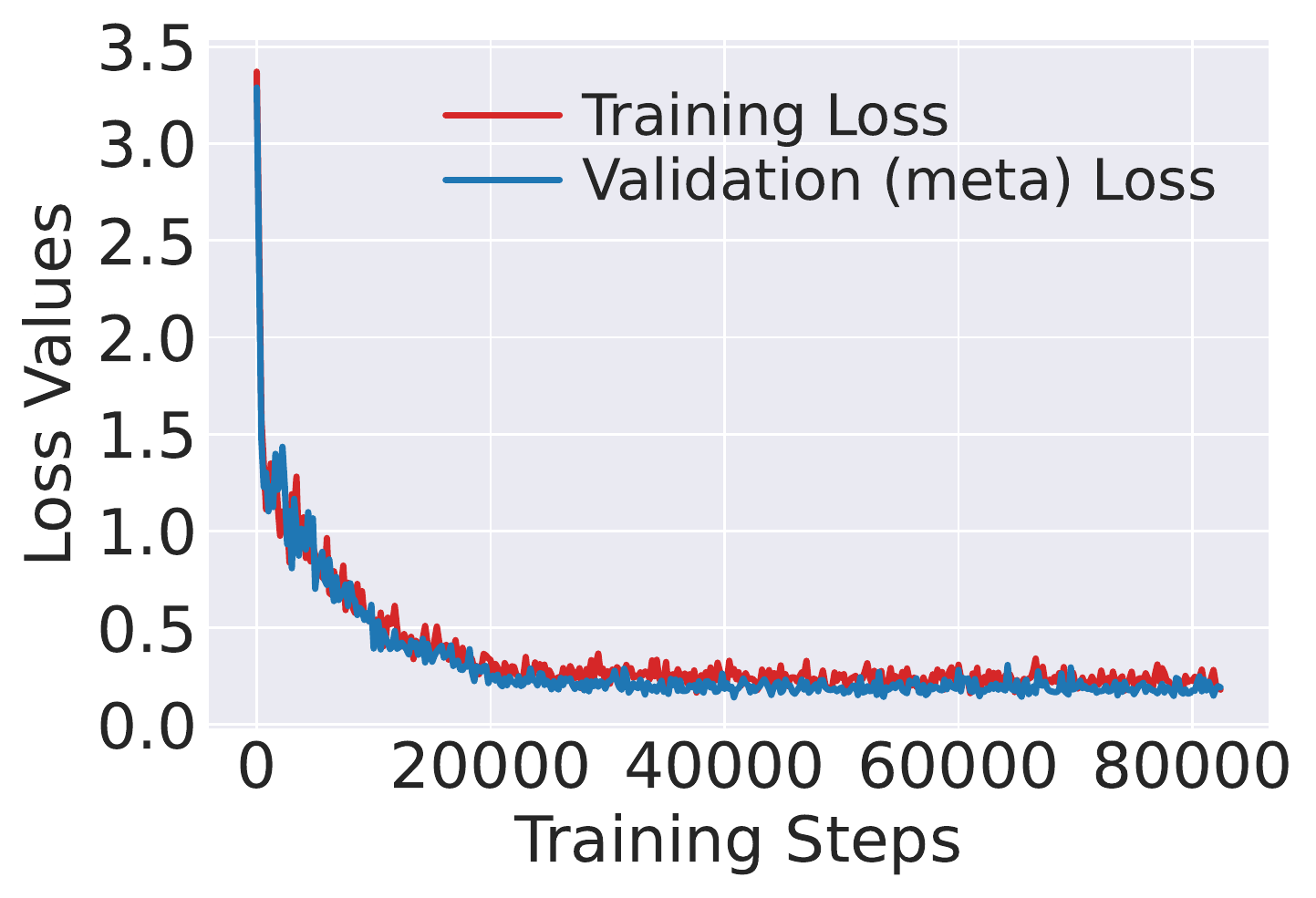}
    \caption{Training and validation (meta) loss curves over training steps.}
    \label{fig:conv} 
\end{figure}

\section{Error Analysis}

We further investigate the error rate with respect to each slot. We adopt SOM-DST as the primary model and compare scheme S2 to ASSIST with the best value of $\alpha$ ($\alpha=0.4$) and ASSIST without using the pseudo labels ($\alpha=0.0$). We conduct the experiment on MultiWOZ 2.4 as well and the results are illustrated in Figure~\ref{fig:error}. It is shown that MetaASSIST achieves lower error rates for 28 slots when compared to ASSIST ($\alpha=0.0$). MetaASSIST also outperforms ASSIST ($\alpha=0.4$) on 18 of the 30 slots. These results verify again the superiority of our proposed framework MetaASSIST. 

\begin{figure*} [t]
    \centering    \includegraphics[width=1.0\linewidth]{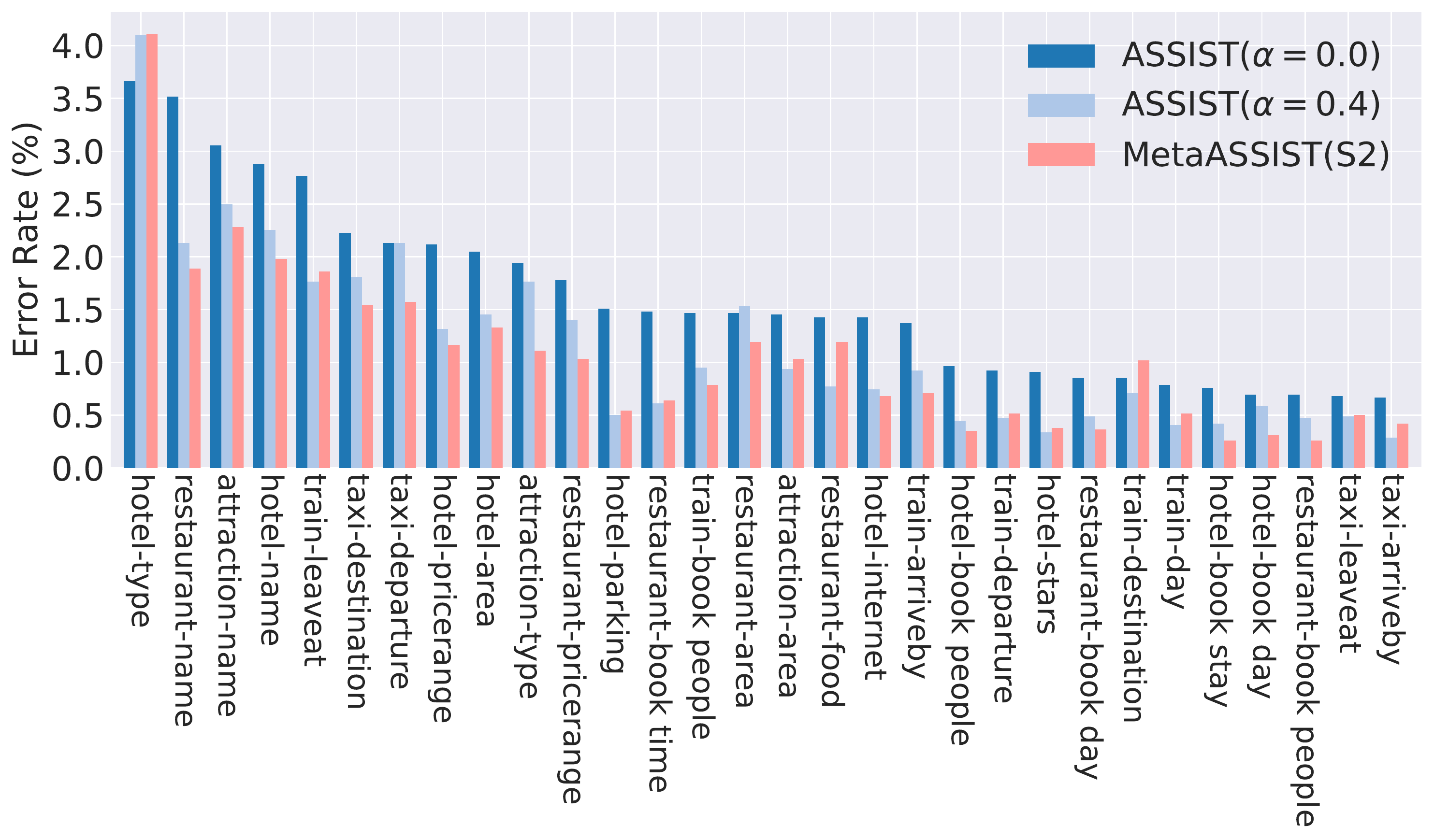}
    \caption{The error rate of each slot on MultiWOZ 2.4. SOM-DST is employed as the primary model.}
    \label{fig:error} 
\end{figure*}

\end{document}